\documentclass[letterpaper]{article} % DO NOT CHANGE THIS
\usepackage{aaai2026}  % DO NOT CHANGE THIS
\usepackage{times}  % DO NOT CHANGE THIS
\usepackage{helvet}  % DO NOT CHANGE THIS
\usepackage{courier}  % DO NOT CHANGE THIS
\usepackage[hyphens]{url}  % DO NOT CHANGE THIS
\usepackage{graphicx} % DO NOT CHANGE THIS
\usepackage{natbib}  % DO NOT CHANGE THIS AND DO NOT ADD ANY OPTIONS TO IT
\usepackage{caption} % DO NOT CHANGE THIS AND DO NOT ADD ANY OPTIONS TO IT
\usepackage{algorithm}
\usepackage{algorithmic}

\usepackage{newfloat}
\usepackage{listings}

\usepackage{amssymb}
\usepackage{amsmath}
\usepackage{comment}
\usepackage{enumitem}
\usepackage{algorithm}
\usepackage{theoremref}
\usepackage{color}
\usepackage{colortbl}
\usepackage{adjustbox}
\usepackage{multirow}
\usepackage{tablefootnote}
\usepackage{threeparttable}
\usepackage{booktabs}
\usepackage{amsthm}
\usepackage{mathrsfs}
\usepackage{xcolor}
\definecolor{Positive}{rgb}{0,0.5,0}
\definecolor{Negative}{rgb}{0.5,0,0}

\newtheorem{theorem}{Theorem}
\newtheorem*{theorem*}{Theorem}
\newtheorem{definition}{Definition}
\newtheorem{remark}{Remark}[theorem]

\newtheorem*{condition*}{Condition}

\DeclareCaptionStyle{ruled}{labelfont=normalfont,labelsep=colon,strut=off} % DO NOT CHANGE THIS
\floatstyle{ruled}
\newfloat{listing}{tb}{lst}{}
\floatname{listing}{Listing}
\title{A Flat Minima Perspective on Understanding Augmentations and Model Robustness}
\author{
    Weebum Yoo\textsuperscript{\rm 1},
    Sung Whan Yoon\textsuperscript{\rm 1,\rm 2}\thanks{Corresponding Author}
}
\affiliations{
    \textsuperscript{\rm 1}Graduate School of Artificial Intelligence, 
    Ulsan National Institute of Science and Technology (UNIST), Republic of Korea\\
    \textsuperscript{\rm 2}Department of Electrical Engineering, 
    Ulsan National Institute of Science and Technology (UNIST), Republic of Korea\\
    \{weebum, shyoon8\}@unist.ac.kr
}
\usepackage{bibentry}
\begin{document}

\maketitle

\begin{abstract}
Model robustness indicates a model's capability to generalize well on unforeseen distributional shifts, including data corruptions and adversarial attacks. 
Data augmentation is one of the most prevalent and effective ways to enhance robustness. Despite the great success of the diverse augmentations in different fields, a unified theoretical understanding of their efficacy in improving model robustness is lacking. 
We theoretically reveal a general condition for label-preserving augmentations to bring robustness to diverse distribution shifts through the lens of flat minima and generalization bound, which de facto turns out to be strongly correlated with robustness against different distribution shifts in practice.
Unlike most earlier works, our theoretical framework accommodates all the label-preserving augmentations and is not limited to particular distribution shifts.
We substantiate our theories through different simulations on the existing common corruption and adversarial robustness benchmarks based on the CIFAR and ImageNet datasets.
\end{abstract}
\begin{links}
    \link{Code}{https://github.com/pyoo96/aug-flatmin-robustness}
\end{links}


\section{Introduction}
\label{sec:intro}
\textit{Model robustness}, which is a critical factor of deep models in applications requiring high reliability, such as autonomous vehicles and medical diagnosis, entails maintaining performance against data distribution shifts.
In the past decade, data augmentation has been widely used as a popular and pragmatic technique to enhance the model performance, as well as the robustness against data corruption, adversarial attacks, or even domain shifts \cite{augment_survey, ME-ADA, PAC_DG1, DeepAug10:2021}.
The intuition of its efficacy relies on the belief that augmentations enrich the training data distribution, which allows models to easily extrapolate to unseen data distributions, which is the so-called generalization capability.

Despite the utility of augmentations, a unified theory formally clarifying how augmentations can enhance model robustness against diverse distribution shifts has not been well established. 
Prior analyses are mostly confined to either particular augmentations or adversarial robustness~\cite{ME-ADA, AugmentNAdvGood, NoLabelNAdvGood1, aug_analysis_adv1}. 
Although some studies step forward from the aforementioned works, their analyses rely solely on empirical observations without a theoretical rationale, or the focus is exclusively on certain categories of distribution shifts~\cite{zhang2024duality, aug_analysis_custom1, aug_analysis_dg}.

In this paper, we offer a series of theoretical insights, including a sufficient condition (termed the PSA condition), that explains how label-preserving data augmentations can bolster model robustness against general distributional shifts through the lens of flat minima.
Our analysis has two main branches: \textbf{i)} First, we mathematically bridge the general form of label-preserving augmentations to the improved generalization bound. 
To give a brief sketch of our development, we start to demonstrate the equivalence between the input space region covered by the augmented samples and the corresponding parameter space region with the same loss values (formalized by Theorem~\ref{thm1}). 
Based on the equivalence, we then claim that the augmentation satisfying the PSA condition flattens the loss surface on the parameter space (Theorem~\ref{thm2}), finally reaching to the improved generalization bound against distribution shifts via leveraging the flattened loss surface (Theorem~\ref{thm3}).
\textbf{ii)} Next, we validate our theoretical findings by evaluating existing augmentation methods across different robustness benchmarks, encompassing data distribution shifts caused by common corruptions and adversarial attacks.
Our findings show that when augmentations have non-negligible sample coverage near the original image--which aligns with the PSA condition--they consistently enhance model robustness. In contrast, when augmentations fail to improve robustness, they exhibit negligible density near the original image, indicating that PSA serves as both a sufficient condition for robustness and a factor highly correlated with robustness.

\begin{figure*}[t]
	\begin{center}
		\includegraphics[width=1.0\linewidth]{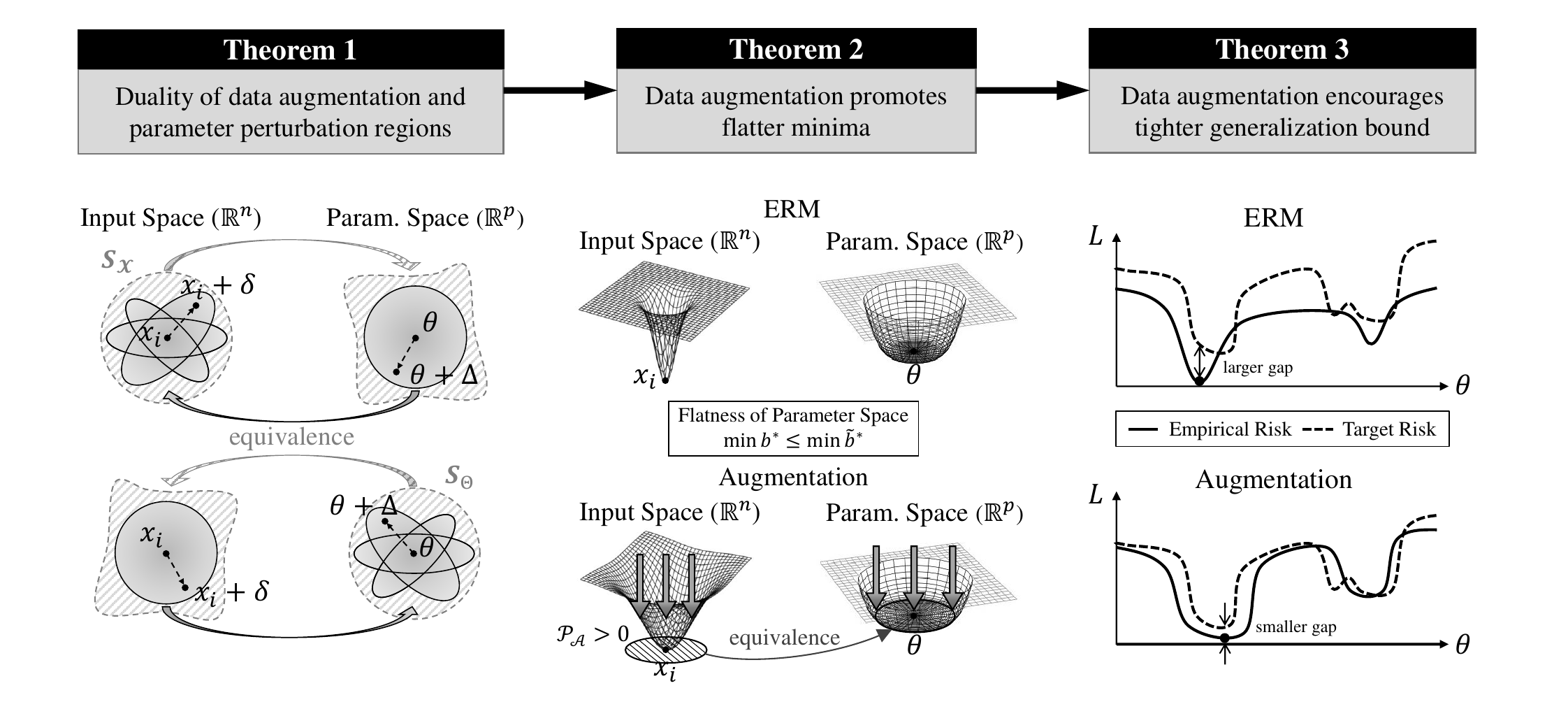}
	\end{center}
	\caption{A conceptual overview of how augmentations can relate to model robustness. Theorem~\ref{thm1} shows that the neighboring region of an input datapoint can be mapped to an equivalent region in the parameter space (and vice versa). Theorem~\ref{thm2} demonstrates that ensuring a nonzero probability for an augmentation within this neighborhood (PSA condition) promotes flatter minima in parameter space by effectively regulating the proximal regions of both the input and parameter spaces. Finally, Theorem~\ref{thm3} presents a generalization bound for models trained with augmentations, insisting that augmentations with dense representations near the original image foster a tighter generalization gap.}
\label{fig:sketch}
\end{figure*}

\section{Preliminaries}
\label{preliminaries}

\subsection{Basic Notations} % Slightly shortened paragraph

Let us consider an input $x\in\mathbb{R}^{n}$ from input space $\mathcal{X}$, which is paired with a target label $y\in\mathbb{R}^{c}$ from label space $\mathcal{Y}$, where $n$ and $c$ are the dimensions of the input space and the label space.
A model $f(\cdot;\theta): \mathbb{R}^{n}\rightarrow\mathbb{R}^{c}$ parameterized by $\theta\in\mathbb{R}^{p}$ maps a given input to the estimated label, where $p$ is the dimension of the model parameter space.
The loss function $\mathcal{L}(\cdot, \cdot)$ quantifies how far the predicted label $f(x;\theta)$ deviates from the true label $y$, expressed as $\mathcal{L}\big(f(x;\theta), y\big)$. Given a dataset $\{ (x_i, y_i)  \}_{i=1}^N$ of $N$ samples drawn from data distribution $\mathcal{D}$, the true risk and the empirical risk are:
\begin{gather}
\mathcal{E}_{\mathcal{D}}(\theta):=\mathbb{E}_{(x,y)\sim \mathcal{D}}\big[\mathcal{L}(f(x;\theta), y)\big] \\
\hat{\mathcal{E}}_{\mathcal{D}}(\theta):=\frac{1}{N}\sum_{i=1}^{N}\mathcal{L}(f(x_i;\theta), y_i) 
\end{gather}
respectively, where $\mathcal{E}_{\mathcal{D}}(\theta)$ and $\hat{\mathcal{E}}_{\mathcal{D}}(\theta)$ denote the true and the empirical risk.

\subsubsection{Data Augmentation}
Data augmentation $\mathcal{A}(\cdot):\mathbb{R}^{n}\rightarrow\mathbb{R}^{n}$ augments a given input $x$ to augmented input $\tilde{x}:=\mathcal{A}(x)$.
Let us further represent it with the difference between the original input and the augmented one:  $\tilde{x}:=\mathcal{A}(x)=x+\delta$, where $\delta\in\mathbb{R}^{n}$.
\footnote{By formulating augmentation as an additive perturbation, we transform the perturbation on the parameter space to one on the input space. Details are in the following section.}
Also, we define $\mathcal{P}_{\mathcal{A}}(\tilde{x}|x)$ as the probability density function of an augmented sample $\tilde{x}$ given $x$. Finally, we let $\tilde{\mathcal{D}} := \{ (\mathcal{A}(x), y) : (x, y) \sim \mathcal{D} \}$ represent the augmented dataset.

\subsection{Model Flatness}

\subsubsection{Definition} \textit{Model flatness} characterizes the extent of change in the model's loss values across proximate points in the parameter space. When the loss rapidly changes around the found minima, it indicates that the model is located at \textit{sharp minima}. Otherwise, it denotes \textit{flat minima} when the loss varies smoothly.
The change of losses around the model parameters can be formalized as follows:
\begin{equation}\label{eq:SAM1}
\max_{\|\Delta\| \leq \gamma}\mathbb{E}_{(x,y)\sim \mathcal{D}}\big[\mathcal{L}(f(x;\theta + \Delta), y) - \mathcal{L}(f(x;\theta),y)\big],
\end{equation}
where $\Delta\in\mathbb{R}^{p}$ is the perturbation around model parameter $\theta$ that maximally increases loss within a radius $\gamma>0$.

\subsubsection{Sharpness-aware minimization}
The most popular principal way for finding flatter minima is Sharpness-Aware Minimization (SAM) \citep{SAM:2021}, which formally transforms loss minimization into a min-max optimization:
\begin{equation}\label{eq:SAM2}
\min_{\theta}\max_{\|\Delta\| \leq \gamma}\mathbb{E}_{(x,y)\sim \mathcal{D}}[\mathcal{L}(f(x;\theta + \Delta), y)], 
\end{equation}
As formulated in the work of \citep{SAM:2021}, when adding the loss value at $\theta$, only the first term of \eqref{eq:SAM1} remains, yielding the objective function in \eqref{eq:SAM2}.
Through the minimization of the maximization term, SAM aims to find the minima $\theta$ with flatter loss surfaces around $\gamma$ radius.

\section{How Can Augmentations Enhance Robustness?}\label{sec:3}

In this section, we provide a rigorous link between \textit{data augmentations} and \textit{improved generalization capability}. Our main claims are twofold: \textbf{i)} Translation between equivalent input manifold and parameter manifold, \textbf{ii)} Association of flatter loss in the input space via augmentations with parameter space flatness, leading to a reduced generalization gap.

%%%%%%%%%%%%%%%%%%%%%%%%%%%%%%%%%%%%%%%%%%%%%%%%%%
%%%%%%%%%%%%%%%%%%%%%%%%%%%%%%%%%%%%%%%%%%%%%%%%%%

\subsection{Translation Between Input and Parameter Perturbations}

Consider a perturbation $\delta \in \mathbb{R}^{n}$ around an input vector $x$. Intuitively, there should be a corresponding perturbation $\Delta \in \mathbb{R}^{p}$ around the parameter $\theta$ that yields the same output, i.e.,\ $f(x + \delta; \theta) = f(x; \theta + \Delta)$. Moreover, for any perturbation in the closed ball $\{\delta \in \mathbb{R}^n : \|\delta\| \leq \gamma\}$ in the input space, we hypothesize an \emph{equivalent} or \emph{compensatory} region $\mathcal{C}_\Theta$ in the parameter space---that is, for every $\|\delta\| \leq \gamma$, there is a matching parameter perturbation $\Delta \in \mathcal{C}_\Theta$ producing the same output. 

Conversely, the same intuition in reverse applies to parameter perturbations. Consider a perturbation $\Delta \in \mathbb{R}^{p}$ around $\theta$. For any $\|\Delta\|\le \gamma$, we posit an analogous \emph{compensatory} region $\mathcal{C}_\mathcal{X}$ in the input space such that $f(x;\theta + \Delta) = f(x + \delta; \theta)$. Specifically, there is some input shift $\delta \in \mathcal{C}_\mathcal{X}$ that replicates the same output change induced by $\Delta$. 

Our goal is to formalize this duality of perturbations in both the input and parameter spaces. We now present the notion of \emph{functional compensatory sets}, which identifies the precise sets of input shifts needed to compensate for every parameter perturbation of norm up to $\gamma$, and vice versa.

\subsubsection{Functional Compensatory Sets}

\begin{definition}(Functional Compensatory Sets)\label{def:compensatory} \medskip \\
\noindent
\textbf{(A) Parameter-to-Input.} 
Given a dataset $\mathcal{D}=\{x_i\}_{i=1}^N$ and a parameter perturbation set
$ \{\Delta \in \mathbb{R}^p : \|\Delta\|\le \gamma\}, $
a \emph{functional compensatory set} $\mathcal{C}_\mathcal{X}$ in the input space satisfies:
\begin{gather}
\text{For every} \: x \in \mathcal{D} \text{ and } \| \Delta \| \leq \gamma, \text{ there exists } \delta \in \mathcal{C}_\mathcal{X} \text{ s.t. } \nonumber\\
f(x;\theta+\Delta) = f(x+\delta;\theta).\nonumber
\end{gather}
This set $\mathcal C_\mathcal X$ is not necessarily unique, and it characterizes all possible input shifts that replicate the effect of parameter perturbations within norm $\gamma$.

\medskip
\noindent
\textbf{(B) Input-to-Parameter.} 
Given a dataset $\mathcal{D}=\{x_i\}_{i=1}^N$ and an input perturbation set 
$ \{\delta \in \mathbb{R}^n : \|\delta\|\le \gamma\}, $
a \emph{functional compensatory set} $\mathcal{C}_\Theta$ in the parameter space satisfies: 
\begin{gather}
\text{For every} \: x \in \mathcal{D} \text{ and } \| \delta \| \leq \gamma, \text{ there exists } \Delta \in \mathcal{C}_\Theta \text{ s.t. } \nonumber\\
f(x+\delta;\theta) = f(x;\theta+\Delta).\nonumber
\end{gather}
This set $\mathcal{C}_\Theta$ need not be unique, and captures all possible parameter shifts that replicate the effect of any input perturbation within norm $\gamma$.
\end{definition}

\subsubsection{Formalizing the Input–Parameter Perturbation Duality}

The following result shows that an $n$-dimensional ball in the input space, whose radius depends on dataset $\mathcal{D}$ and the Jacobian's singular values, can compensate for any parameter perturbation $\| \Delta \| \leq \gamma$. Conversely, the same analysis guarantees that a $p$-dimensional ball in the parameter space, determined by $\mathcal{D}$ and the Jacobian, can compensate for any input perturbation $\| \delta \| \leq\ \gamma$.

Let $\mathbf{J}$ be the Jacobian of a function $f: \mathbb{R}^n \times \mathbb{R}^p \rightarrow \mathbb{R}^c$, where $\mathbb{R}^n, \mathbb{R}^p,$ and $\mathbb{R}^c$ are the input, parameter, and output spaces. The Jacobian matrix $\mathbf{J} \in \mathbb{R}^{c \times (n + p)}$ can be decomposed into its input side and the parameter side partial derivatives, i.e. $\mathbf{J} = \left[\mathbf{J_x} \ \ \mathbf{J}_{\boldsymbol{\theta}} \right]$, where $\mathbf{J_x} \in \mathbb{R}^{c \times n}$ and $\mathbf{J}_{\boldsymbol{\theta}} \in \mathbb{R}^{c \times p}$. We will use the notation $\mathbf{J_x}(x_i)$ and $\mathbf{J}_{\boldsymbol{\theta}}(\theta)$ when we need to explicitly represent the evaluation of $\mathbf{J_x}$ and $\mathbf{J}_{\boldsymbol{\theta}}$ at point $x_i \in \mathbb{R}^n$ and $\theta \in \mathbb{R}^p$, respectively.  Finally, let $\sigma_{\min}^{x_i}$ and $\sigma_{\max}^\theta$ represent the mininum and the maximum values of the matrix $\mathbf{J_x}(x_i)$ and $\mathbf{J}_{\boldsymbol{\theta}}(\theta)$.

Grounded on an o-minimal structure and the lazy training phenomenon of conventional neural networks, we utilize Sard's theorem~\cite{Sard, MorseSard} together with Taylor expansion to build the functional compensatory sets. We construct two sets---$\mathcal{S}_\mathcal{X}$ and $\mathcal{S}_\Theta$---that fulfill the definition of functional compensatory sets $\mathcal{C}_\mathcal{X}$ and $\mathcal{C}_\Theta$ as follows:

\smallskip

\begin{theorem}\label{thm1}
(Input--Parameter Duality via Functional Compensatory Sets) 
Consider a function $f: \mathbb{R}^n \times \mathbb{R}^p \rightarrow \mathbb{R}^c$.

\smallskip
\noindent
\textbf{(A) Parameter-to-Input}
Define the ball of parameter perturbations
$\{\Delta \in \mathbb{R}^p : \|\Delta\|\le \gamma\}$ and 
$$
\mathcal{S}_\mathcal{X} \;=\; \left\{\delta \in \mathbb{R}^n : \|\delta\|\;\le\;\max_{x \in \mathcal{D}}\!\Bigl(\dfrac{\sigma_{\max}^\theta}{\sigma_{\min}^{x}}\Bigr)\,\gamma \right\}.
$$
Then, $\mathcal{S}_\mathcal{X}$ is always a functional compensatory set for the parameter perturbation ball and dataset $\mathcal{D}$.\\
\textbf{(B) Input-to-Parameter}
Conversely, define the ball of input perturbations
$\{\delta \in \mathbb{R}^n : \|\delta\|\le \gamma\}$ and 
$$\mathcal{S}_\Theta \;=\; \left\{\Delta \in \mathbb{R}^p : \|\Delta\|\;\le\;\max_{x \in \mathcal{D}}\!\Bigl(\dfrac{\sigma_{\max}^{x}}{\sigma_{\min}^\theta}\Bigr)\,\gamma \right\}.$$
Then, $\mathcal{S}_\Theta$ is always a functional compensatory set for the input perturbation ball and dataset $\mathcal{D}$.
\end{theorem}

\begin{proof}
See Appendix B.1.
\end{proof}

\begin{remark}
\textbf{(Loss Stability Across Input and Parameter Spaces)} 
Ensuring the model’s loss is low on all datapoints within a small neighborhood around \(x\) (i.e.,\ for all \(x + \delta\) with \(\delta \in \mathcal{S}_\mathcal{X}\)) guarantees that the corresponding neighborhood in parameter space \(\{\theta + \Delta : \|\Delta\| \le \gamma\}\) also has low loss. Conversely, if the model’s loss is low within a small neighborhood around the parameters (i.e.\ \(\theta + \Delta\) for \(\Delta \in \mathcal{S}_\Theta\)), then any corresponding input perturbation \(\delta\) with \(\|\delta\|\le \gamma\) also yields low loss.
Hence, stability in the input space induces stability in the parameter space, and vice versa.
\end{remark}

\begin{table*}[ht]
\centering
\begin{adjustbox}{width=1.0\textwidth}
\begin{tabular}{c|ccccc|ccccc}
\midrule[1.5pt]
\multirow{2}{*}{\ \textbf{Distance}\ } &
\multicolumn{5}{c|}{\textbf{CIFAR-10}} &
\multicolumn{5}{c}{\textbf{CIFAR-100}}\\
& \ AugMix \ & \ RandAug \ & \ PixMix \ & \ DeepAug \ & \ StyleAug \
& \ AugMix \ & \ RandAug \ & \ PixMix \ & \ DeepAug \ & \ StyleAug \ \\
\midrule[1.5pt]

 $\gamma=0.01$ &
0.0025 & 0.0064 & 0.0014 & \textit{0.0000} & \textit{0.0000} &
0.0024 & 0.0060 & 0.0011 & \textit{0.0000} & \textit{0.0000} \\

 $\gamma=0.05$ &
\textbf{0.0100} & 0.0064 & 0.0015 & \textit{0.0000} & \textit{0.0000} &
\textbf{0.0107} & 0.0060 & 0.0011 & \textit{0.0000} & \textit{0.0000} \\

 $\gamma=0.10$ &
\textbf{0.0205} & 0.0065 & 0.0015 & \textit{0.0000} & \textit{0.0000} &
\textbf{0.0202} & 0.0061 & 0.0012 & \textit{0.0000} & \textit{0.0000} \\

 $\gamma=0.50$ &
\textbf{0.0959} & \textbf{0.0391} & 0.0015 & 0.0001 & \textit{0.0000} &
\textbf{0.0993} & \textbf{0.0378} & 0.0012 & 0.0001 & \textit{0.0000} \\

\midrule[1.5pt]
\end{tabular}
\end{adjustbox}

\caption{Empirical CDF $\big(F_N(\gamma)\big)$ for selected distance thresholds $(\gamma)$ on CIFAR datasets, illustrating sample density near the original image over different augmentation methods.
The higher the eCDF value, the stronger an augmentation satisfies PSA condition. 
Sorted in descending order. (\textit{Italic} entries have near-zero probability $(< 5 \cdot 10^{-5})$, and \textbf{bold} entries exceeds 0.01.)
}
\label{tab:eCDF_cifar}
\end{table*}

\subsection{Linking Augmentations to Model Flatness and Generalization}

Built upon the theorems above, we now formalize how augmentations can lead to a flatter loss landscape around the minima and the improved generalization capability.
Before providing the details, let us rephrase the formal definition of flat minima, which is called $b$-flat minima \cite{shi2021overcoming}:
\begin{definition}
($b$-flat local minima) Given loss $\mathcal{L}(\cdot,\cdot)$ and dataset $\mathcal{D}$, a model parameter $\theta\in\mathbb{R}^{p}$ is $b$-flat minima if the followings hold for the perturbation on parameter $\Delta\in\mathbb{R}^{p}$:
\begin{align}
&\text{For all} \: \|\Delta\| \leq b, \:\: \hat{\mathcal{E}}_{\mathcal{D}}(\theta+\Delta) = \hat{\mathcal{E}}_{\mathcal{D}}(\theta) \text{ and } \nonumber \\
&\text{For some} \: \|\Delta\| > b, \:\: \hat{\mathcal{E}}_{\mathcal{D}}(\theta+\Delta) > \hat{\mathcal{E}}_{\mathcal{D}}(\theta) . \nonumber
\end{align}
\end{definition}

For input $x$, an augmentation $\mathcal{A}$ induces a distribution of variants $\tilde{x}$ near $x$.
Let $\gamma_{\mathcal{A}}$ be the largest radius of a ball around $x$ with nonzero probability density. If $\gamma_{\mathcal{A}}>0$, we say $\mathcal{A}$ covers the local neighborhood.
We now formalize this property and later show that it is critical for robustness.

\begin{condition*} \label{assumption1}
(Proximal-Support Augmentation (PSA))\\
\indent Given $x\in\mathbb{R}^{n}$ and augmentation $\mathcal{A}(\cdot)$, the probability density function $\mathcal{P}_{\mathcal{A}}(\tilde{x}|x)$ satisfies:
\begin{equation}
\text{For all} \: \|\delta\|\leq\gamma_\mathcal{A}, \: \mathcal{P}_{\mathcal{A}}(\tilde{x}|x)>0,
\end{equation}
\textit{where $\delta=\tilde{x}-x$ and $\gamma_\mathcal{A}$ is some positive real number.}
\end{condition*}

PSA is a condition that covers any label-preserving augmentation and ties it to flat minima. It serves as a testable rule by setting a local coverage requirement around $x$. Discarding dependencies on the model’s interpolation ability between data points, PSA is a strong condition for zero loss over all neighborhoods of inputs for any architectures. % on PSA

Let $\gamma_{\mathcal{A}}$ be the radius value for PSA, and $\tilde{\mathcal{D}}$ represent the augmented dataset with respect to $\mathcal{D}$ under PSA condition.
Let $\Theta^{*}$ and $\tilde{\Theta}^{*}$ be the sets of the optimal model parameters whose elements $\theta^*\in\Theta^*$ and $\tilde{\theta}^*\in\tilde{\Theta}^*$ satisfy the following equalities, $\mathcal{E}_\mathcal{D}(\theta^{*}) = 0$ and $\mathcal{E}_\mathcal{\tilde{D}}(\tilde{\theta}^{*}) = 0$, respectively. 
Then our claim is that the minimum $b$-flatness among the solution parameters in $\tilde{\Theta}^{*}$, shows large $b$ (flatter) than the minimum $b$-flatness among the solutions in $\Theta^{*}$, which is trained on $\mathcal{D}$:

\begin{theorem}\label{thm2}
(Flatness of $\tilde{\theta}^{*}$) Let $\theta^*\in\Theta^*$ and $\tilde{\theta}^{*}\in\tilde{\Theta}^*$ be $b^*$ and $\tilde{b}^*$-flat minima, respectively. The following inequality holds: 
\begin{equation}
\min_{\theta^*\in\Theta^*}b^{*}\leq \min_{\tilde{\theta}^*\in\tilde{\Theta}^*}\tilde{b}^{*}. \nonumber
\end{equation}
\end{theorem}

\begin{proof}
See Appendix B.2.
\end{proof}

\begin{remark}
\textbf{(Any augmentation with nearby representations around the original promotes flatter minima)}
The key understanding of the theorem above is that if augmentations cover the ball-shaped region around the given original sample with radius $\gamma_\mathcal{A}$, then the model $\tilde{\theta}^{*}$ suppresses the loss values of the region. Next, the flat-region on the input space is translated to the functional compensatory set region on parameter space, which at least contains $\{\delta : \| \delta \| \leq \left(\max_{x \in \mathcal{D}} \sigma_{\max}^\theta / \sigma_{\min}^x \right)^{-1} \cdot \gamma_\mathcal{A} \}$.
\end{remark}

The final linkage to the generalization capability is straightforward by relying on the prior theoretical results that bridge Robust Risk Minimization (RRM) and the generalization bound \cite{SWAD:2021}. As RRM targets flat minima, Theorem~\ref{thm2} shows that augmentations further flatten the optimum’s neighborhood, and directly tightens the bound in Theorem~\ref{thm3}.
Let us define $\gamma_{\Theta}:=\left(\max_{x \in \mathcal{D}} \sigma_{\max}^\theta / \sigma_{\min}^x \right)^{-1} \cdot \gamma_\mathcal{A}$, then the following theorem holds:
\begin{theorem}\label{thm3}
(Generalization bound) Given $M$ covering sets $\{\Theta_k\}_{k=1}^M$ of parameter space $\Theta$ with $\Theta = \bigcup_{k=1}^M \Theta_k$ and $\text{diam}(\Theta) = \sup_{\theta,\theta' \in \Theta} ||\theta - \theta'||_2$, where $M = \left\lceil \frac{\text{diam}(\Theta)}{\gamma_\Theta} \right\rceil^p$, and VC dimension $v_k$ for each $\Theta_k$, the following inequality holds with probability at least $1-\delta$:
\begin{align}
\mathcal{E}_{\mathcal{T}}(\tilde{\theta}^*) &< \hat{\mathcal{E}}_{\mathcal{\tilde{D}}}(\tilde{\theta}^*) + \frac{1}{2} \textbf{Div}(\mathcal{D}, \mathcal{T}) \\ &+ \max_{k\in [1,M]} \left[ \sqrt{\frac{v_k \log(N/v_k)}{N} + \frac{\log(M/\delta)}{N}} \right], \notag
\end{align}
where $\textbf{Div}(\mathcal{D}, \mathcal{T}) = 2 \sup_A |\mathbb{P}_\mathcal{D}(A) - \mathbb{P}_\mathcal{T}(A)|$ measures the maximal discrepancy between the source and the target distributions $\mathcal{D}$ and $\mathcal{T}$, and $N$ is the number of samples drawn from $\mathcal{\tilde{D}}$.
\end{theorem}
\begin{proof}
See Appendix B.3.
\end{proof}
\begin{remark}
\textbf{(Augmentations can improve generalization against data distribution shifts)}
The theorem implies that the minimization of empirical loss for the augmented dataset, i.e., $\hat{\mathcal{E}}_{\mathcal{\tilde{D}}}(\tilde{\theta}^*)$, directly aims the tighter generalization bound on target distribution $\mathcal{T}$. Also, the term $M$ is related to $\gamma_\mathcal{A}$, which measures how augmented data distribution covers a wider range (referred to PSA condition.) When augmentation covers a wider range around the original sample, i.e., a larger $\gamma_\mathcal{A}$, it suppresses $M$, leading to the smaller last term of generalization bound. For trivial $\gamma_\mathcal A \simeq 0$, the flatness and bound tightness gains vanish. % on PSA
\end{remark}

\begin{table*}[ht]
\centering
\begin{adjustbox}{width=0.98\textwidth}
\begin{tabular}{cc|c|ccccc}
\midrule[1.5pt]
\textbf{\ \ Dataset \ \ } & \textbf{ \ \ Metrics \ \ } & $ \ \ \ $ ERM $ \ \ \ $ & $ \quad \ $ AugMix $ \quad \ $ & $ \quad \ $ RandAug $ \quad \ $ & $ \quad \ $ PixMix $ \quad \ $ & $ \quad \ $ DeepAug $ \quad \ $ & $ \quad \ $ StyleAug $ \quad \ $ \\
\midrule[1.5pt]

\multirow{3}{*}{CIFAR-10} 
& $\mu_{\text{PAC-Bayes}}\downarrow$ & 
168.83 &
\textbf{117.16} {\scriptsize (-51.67)} &
\textbf{110.80} {\scriptsize (-58.03)} &
\textbf{102.87} {\scriptsize (-65.96)} &
\textbf{67.74} {\scriptsize (-101.09)} &
\textbf{122.80} {\scriptsize (-46.03)} \\

& LPF $\downarrow$ &
1.43 &
\textbf{0.54} {\scriptsize (-0.89)} &
\textbf{0.45} {\scriptsize (-0.98)} &
\textbf{0.37} {\scriptsize (-1.06)} &
\textbf{0.72} {\scriptsize (-0.71)} &
\textbf{1.13} {\scriptsize (-0.30)} \\

& $\epsilon_{\text{sharp}}\downarrow$ & 
40.90 & 
\textbf{25.16} {\scriptsize (-15.74)} &
\textbf{24.47} {\scriptsize (-16.43)} &
\textbf{26.57} {\scriptsize (-14.33)} &
{55.63} {\scriptsize (+14.73)} &
{130.64} {\scriptsize (+89.74)} \\

\midrule

\multirow{3}{*}{CIFAR-100} 

& $\mu_{\text{PAC-Bayes}}\downarrow$ & 
181.10 &
\textbf{155.74} {\scriptsize (-25.36)} &
\textbf{149.79} {\scriptsize (-31.31)} &
\textbf{134.58} {\scriptsize (-46.52)} &
\textbf{72.24} {\scriptsize (-108.86)} &
\textbf{129.02} {\scriptsize (-52.08)} \\

& LPF $\downarrow$ & 
2.37 &
\textbf{1.84} {\scriptsize (-0.53)} &
\textbf{1.69} {\scriptsize (-0.68)} &
\textbf{1.42} {\scriptsize (-0.95)} &
2.37 {\scriptsize (0.00)} &
{3.95} {\scriptsize (+1.58)} \\

& $\epsilon_{\text{sharp}}\downarrow$ & 
42.39 &
\textbf{39.32} {\scriptsize (-3.07)} &
\textbf{37.76} {\scriptsize (-4.63)} &
\textbf{33.14} {\scriptsize (-9.25)} &
{84.48} {\scriptsize (+42.09)} &
{302.81} {\scriptsize (+260.42)} \\

\midrule[1.5pt]
\multicolumn{3}{c|}{\textbf{Improvement Rate}} & \textbf{6/6} & \textbf{6/6} & \textbf{6/6} & 3/6 & {2/6} \\
\midrule[1.5pt]
\end{tabular}
\end{adjustbox}

\caption{
Flatness metrics of ERM and different augmentations on CIFAR datasets.
Overall, when augmentation meets the PSA condition (nonzero probability near the original image) strongly, the flatter minima are found consistently. (\textit{$\downarrow$: The lower the better, i.e., flatter minimum.})
}
\label{tab:flatness_metric_cifar}
\end{table*}

\begin{table*}[t]
\centering
\begin{adjustbox}{width=0.96\textwidth}
\begin{tabular}{c|c|ccccc}
\midrule[1.5pt]
\textbf{$\quad \ $ Benchmarks $\quad \ $} & \ \ ERM \ \ & 
$\quad \ $ AugMix $\quad \ $ & $\quad \ $ RandAug $\quad \ $ & $\quad \ $ PixMix $\quad \ $ & $\quad \ $ DeepAug $\quad \ $ & $\quad \ $ StyleAug $\quad \ $ \\
\midrule[1.5pt]

CIFAR-10-C $\downarrow$
& 30.54
& \textbf{15.24} {\scriptsize (-15.30)}
& \textbf{19.65} {\scriptsize (-10.89)}
& \textbf{10.60} {\scriptsize (-19.94)}
& {31.92} {\scriptsize (+1.38)}
& \textbf{30.50} {\scriptsize (-0.04)} \\

CIFAR-10-$\overline{\text{C}}$ $\downarrow$
& 31.35
& \textbf{20.28} {\scriptsize (-11.07)}
& \textbf{20.64} {\scriptsize (-10.71)}
& \textbf{14.60} {\scriptsize (-16.75)}
& {36.75} {\scriptsize (+5.40)}
& {36.57} {\scriptsize (+5.22)} \\

CIFAR-100-C $\downarrow$
& 59.04
& \textbf{42.64} {\scriptsize (-16.40)}
& \textbf{46.59} {\scriptsize (-12.45)}
& \textbf{35.20} {\scriptsize (-23.84)}
& {62.34} {\scriptsize (+3.30)}
& {70.91} {\scriptsize (+11.87)} \\

CIFAR-100-$\overline{\text{C}}$ $\downarrow$
& 62.43
& \textbf{48.38} {\scriptsize (-14.05)}
& \textbf{48.32} {\scriptsize (-14.11)}
& \textbf{40.20} {\scriptsize (-22.23)}
& {67.91} {\scriptsize (+5.48)}
& {76.56} {\scriptsize (+14.13)} \\

CIFAR-10, $L_2$ $\downarrow$
& 77.61
& \textbf{70.76} {\scriptsize (-6.85)}
& \textbf{76.15} {\scriptsize (-1.46)}
& \textbf{65.81} {\scriptsize (-11.80)}
& {91.18} {\scriptsize (+13.57)}
& {91.06} {\scriptsize (+13.45)} \\

CIFAR-10, $L_\infty$ $\downarrow$
& 98.49
& {99.10} {\scriptsize (+0.61)}
& {99.86} {\scriptsize (+1.37)}
& {99.46} {\scriptsize (+0.97)}
& {100.00} {\scriptsize (+1.51)}
& {99.98} {\scriptsize (+1.49)} \\

CIFAR-100, $L_2$ $\downarrow$
& 98.73
& \textbf{92.76} {\scriptsize (-5.97)}
& \textbf{96.06} {\scriptsize (-2.67)}
& \textbf{90.69} {\scriptsize (-8.04)}
& \textbf{98.44} {\scriptsize (-0.29)}
& {99.69} {\scriptsize (+0.96)} \\

CIFAR-100, $L_\infty$ $\downarrow$
& 99.94
& \textbf{99.67} {\scriptsize (-0.27)}
& 99.94 {\scriptsize (0.00)}
& {99.69} {\scriptsize (-0.25)}
& {99.99} {\scriptsize (+0.05)}
& {99.98} {\scriptsize (+0.04)} \\

\midrule[1.5pt]
\multicolumn{2}{c|}{\textbf{$\qquad \quad$ Improvement Rate $\qquad \quad$}} & \textbf{7/8} & \textbf{6/8} & \textbf{6/8} & {1/8} & {1/8} \\
\midrule[1.5pt]
\end{tabular}
\end{adjustbox}

\caption{Comparison of different augmentation methods against ERM across multiple robustness scenarios over CIFAR datasets, including common corruptions (CIFAR-10/100-C/$\overline{\text{C}}$) and adversarial attacks ($L_2$, $L_{\infty}$). 
In essence, augmentations that closely adhere to the PSA condition enhance model robustness in general.
Conversely, augmentations that do not enhance robustness are ones that fail to satisfy the PSA condition. (\textit{$\downarrow$: The lower the better, i.e., lower error.})
}
\label{tab:main_table_cifar}
\end{table*}

\section{Empirical Examination on Proximal Density, Model Flatness, and Robustness}
\label{sec:experiments}
To empirically validate our theory, we examine proximal density, model flatness, and robustness on different benchmarks. Earlier, we established that any label-invariant augmentation assigning nonzero probability density to the neighborhood of the original image (PSA condition) maps to an analogous region in parameter space (Theorem~\ref{thm1}), which in turn induces flatter minima (Theorem~\ref{thm2}) and yields a tighter generalization bound (Theorem~\ref{thm3}). Therefore, when augmentation populates rich proximal representations for a given dataset (PSA condition strongly met), it should yield flat minima and strong robustness gains under distribution shifts. In contrast, when augmentation fail to bring robustness, the PSA condition will not be satisfied.

\subsection{Augmentation Methods to Be Considered}

Following the taxonomy of~\citet{augment_survey}, augmentations can be categorized into \textit{model-free}, \textit{model-based}, and \textit{policy-based} augmentations. 
Among existing augmentation strategies, we focus on established methods shown to improve robustness or generalization, rather than on the latest or most basic techniques.
For the \textit{model-free} augmentation type, we select \textbf{AugMix} \cite{AugMix9:2021} and \textbf{PixMix} \cite{PixMix:2022}, which are augmentations relying on mixing multiple clean images to the original image. For the \textit{model-based} augmentation type, we choose \textbf{StyleAug} \cite{StyleAug} and \textbf{DeepAug} \cite{DeepAug10:2021}, which utilize image-to-image models to diversify the style of the clean images. Finally, among the \textit{policy-based} augmentation type, we consider \textbf{RandAugment} \cite{RandAugment}, which learns the policy for augmenting images.

\begin{table*}[ht]
\centering
\begin{adjustbox}{width=1.0\textwidth}
\begin{tabular}{c|ccccc|ccccc}
\midrule[1.5pt]
\multirow{2}{*}{\ \textbf{Distance}\ } &
\multicolumn{5}{c|}{\textbf{\textit{tiny}ImageNet}} &
\multicolumn{5}{c}{\textbf{ImageNet}} \\
& \ AugMix \ & \ DeepAug \ & \ PixMix \ & \ RandAug \ & \ StyleAug \
& \ AugMix \ & \ DeepAug \ & \ PixMix \ & \ RandAug \ & \ StyleAug \ \\
\midrule[1.5pt]

 $\gamma=20.0$ &

0.0010 & 0.0008 & 0.0005 & 0.0003 & \textit{0.0000} &
0.0005 & 0.0005 & 0.0002 & 0.0001 & \textit{0.0000} \\

 $\gamma=30.0$ &
0.0054 & 0.0047 & 0.0023 & 0.0014 & \textit{0.0000} &
0.0031 & 0.0028 & 0.0013 & 0.0009 & \textit{0.0000} \\

 $\gamma=40.0$ &
\textbf{0.0182} & \textbf{0.0163} & 0.0084 & 0.0051 & 0.0009 &
\textbf{0.0113} & \textbf{0.0110} & 0.0054 & 0.0034 & 0.0001 \\

 $\gamma=50.0$ &
\textbf{0.0446} & \textbf{0.0399} & \textbf{0.0237} & \textbf{0.0136} & 0.0087 &
\textbf{0.0317} & \textbf{0.0311} & \textbf{0.0169} & \textbf{0.0104} & 0.0008 \\

\midrule[1.5pt]
\end{tabular}
\end{adjustbox}

\caption{Empirical CDF $\big(F_N(\gamma)\big)$ for selected distance thresholds $(\gamma)$ on (\textit{tiny})ImageNet datasets, illustrating sample density near the original image over different augmentation methods.
Once again, higher eCDF values mean stronger PSA condition compliance in augmentations.
Sorted in descending order.
}
\label{tab:eCDF_imagenet}
\end{table*}

\subsection{Tests on Proximal Density}

We first measure each augmentation’s proximal density to the original image over different datasets via observing the empirical cumulative distribution function (eCDF) of $L_2$ distances on the widely used CIFAR and (\textit{tiny})ImageNet datasets~\cite{IN, CIFAR}. While there exist sophisticated visual-semantic distance metrics (e.g., SSIM, LPIPS), they are not input space metrics. Thus, we have chosen $L_2$ distance to reflect our input-parameter space duality. Specifically, we compute how quickly the eCDF accumulates at smaller values, which provides a clear indicator of each augmentation’s density near the original image.
Formally, the eCDF of distances $\{d_i\}_{i=1}^N$ is defined as $F_N(\gamma) = \tfrac{1}{N}\sum_{i=1}^N $$\mathbf{1} \{ d_i$$ \leq \gamma \}$, where $\mathbf{1}\{\cdot\}$ is the indicator. High values in the low-distance region of $F_N(\gamma)$ imply a large fraction of augmented samples remain close to the original image, signaling strong PSA compliance. In contrast, a low accumulation near small $\gamma$ indicates weak adherence.
Our analysis reveals that, for the CIFAR datasets (Table~\ref{tab:eCDF_cifar}), AugMix, RandAug, and PixMix fulfill the PSA condition, whereas DeepAug and StyleAug do not comply, ranked in decreasing order of adherence.
For (\textit{tiny})ImageNet, AugMix, DeepAug, PixMix, and RandAug conform to the PSA condition in descending order, with StyleAug again underperforming.

\subsection{Tests on Model Flatness}

We now explore how an augmentation’s density near the original image affects flatness in model parameters, linking our results from Theorem~\ref{thm1} and~\ref{thm2} with empirical flatness.% metrics in parameter space.

\textbf{Flatness metrics.}
We assess each model using three quantitative measures: PAC-Bayesian flatness measure $\mu_{\text{PAC-bayes}}$~\cite{PACBayes}, local sharpness $\epsilon_{\text{sharp}}$~\cite{eps-sharpness}, and Loss-Pass Filter (LPF) metric, a strong indicator of generalization~\cite{LPF}.
Each of these metrics reflects the difference between the proximal loss around a model parameter and its original loss, which aligns well with our theoretical framework. We therefore avoid purely “pinpoint-based” measures (e.g., maximum Hessian eigenvalues or their traces), as they do not fully capture our theorems’ emphasis on local neighborhoods in parameter space. For details, see Appendix C.2.

\textbf{Experimental results.}
We train WideResNet-40-2 on CIFAR datasets and ResNet18 on \textit{(tiny)}ImageNet dataset, following common conventions. For details on training, see Appendix C.2. Table~\ref{tab:flatness_metric_cifar} and~\ref{tab:flatness_metric_imagenet} report the flatness metrics on CIFAR and \textit{(tiny)}ImageNet. On CIFAR, PSA compliance (Table~\ref{tab:eCDF_cifar} and~\ref{tab:eCDF_imagenet}) is mostly monotone with flatness. Augmentations with larger proximal density (AugMix, RandAug, PixMix) consistently beat ERM, whereas DeepAug and StyleAug give weaker gains. On \textit{(tiny)}ImageNet, the same trend holds but is less ordered---AugMix and PixMix improve all six metrics, DeepAug and RandAug improve most, and StyleAug again brings little benefit, indicating that PSA is a strong but not exclusive factor that drives flatness.

\subsection{Tests on Model Robustness}

We investigate whether augmentations having rich proximal density (firmly fulfilling PSA condition) fosters robustness against different distribution shifts (Theorem~\ref{thm3}).

\textbf{Benchmarks and Metrics}
Robustness denotes a model's capacity to maintain performance under unforeseen data corruptions and perturbations. We employ well-established benchmarks for common corruptions (CIFAR-10/100-C/$\overline{C}$, \textit{tiny}ImageNet-C, ImageNet-C)~\cite{IN-C7:2019,IN-C-bar:2021} and adversarial robustness (untargeted PGD-$L_2$ and $L_\infty$ attacks on CIFAR, \textit{tiny}ImageNet, and ImageNet datasets)~\cite{PGD}. The error values in Table~\ref{tab:main_table_cifar} and~\ref{tab:main_table_imagenet} represent mean corruption errors (mCE) and adversarial errors for the common corruption and adversarial robustness benchmarks, respectively. % Error metrics
Further details on benchmark and metric specifications are in Appendix C.3.

\textbf{Experimental results.} We utilize the same WideResNet-40-2 and ResNet18 models used in flatness measures to evaluate their robustness against distribution shifts (Table~\ref{tab:main_table_cifar} and~\ref{tab:main_table_imagenet}.) Fundamentally, stronger adherence to the PSA condition reveals an augmentation's effectiveness in conferring robustness to diverse distribution shifts, consistent with Theorem~\ref{thm3}. For further experiments including varying backbone networks~\cite{allconvnet, densenet, resnext}, refer to Appendix C.4 and D.

\begin{table*}[ht]
\centering
\begin{adjustbox}{width=0.98\textwidth}
\begin{tabular}{cc|c|ccccc}
\midrule[1.5pt]
\textbf{\ \ Dataset \ \ } & \textbf{ \ \ Metrics \ \ } & $ \ \ \ $ ERM $ \ \ \ $ & $ \quad \ $ AugMix $ \quad \ $ & $ \quad \ $ DeepAug $ \quad \ $ & $ \quad \ $ PixMix $ \quad \ $ & $ \quad \ $ RandAug $ \quad \ $ & $ \quad \ $ StyleAug $ \quad \ $ \\
\midrule[1.5pt]

\multirow{3}{*}{\textit{tiny}ImageNet} 
& $\mu_{\text{PAC-Bayes}}\downarrow$ & 
136.83 &
\textbf{109.18} {\scriptsize (-27.65)} &
\textbf{128.41} {\scriptsize (-8.42)} &
\textbf{106.26} {\scriptsize (-30.57)} &
\textbf{109.36} {\scriptsize (-27.47)} &
\textbf{102.72} {\scriptsize (-34.11)} \\

& LPF $\downarrow$ &
5.95 &
\textbf{3.81} {\scriptsize (-2.14)} &
\textbf{4.96} {\scriptsize (-0.99)} &
\textbf{3.40} {\scriptsize (-2.55)} &
\textbf{4.04} {\scriptsize (-1.91)} &
\textbf{4.33} {\scriptsize (-1.62)} \\

&$\epsilon_{\text{sharp}}\downarrow$ & 
15.35 &
\textbf{13.09} {\scriptsize (-2.26)} &
{25.02} {\scriptsize (+9.67)} &
\textbf{10.72} {\scriptsize (-4.63)} &
{21.30} {\scriptsize (+5.95)} &
{35.40} {\scriptsize (+20.05)} \\

\midrule

\multirow{3}{*}{ImageNet} 

& $\mu_{\text{PAC-Bayes}}\downarrow$ & 
219.22 &
\textbf{210.63} {\scriptsize (-8.59)} &
\textbf{211.41} {\scriptsize (-7.81)} &
\textbf{187.89} {\scriptsize (-31.33)} &
\textbf{215.58} {\scriptsize (-3.64)} &
{235.40} {\scriptsize (+16.18)} \\

& LPF $\downarrow$ & 
8.70 &
\textbf{8.12} {\scriptsize (-0.58)} &
\textbf{8.67} {\scriptsize (-0.03)} &
\textbf{8.22} {\scriptsize (-0.48)} &
{8.76} {\scriptsize (+0.06)} &
{9.27} {\scriptsize (+0.57)} \\

&$\epsilon_{\text{sharp}} \downarrow$ & 
36.80 &
\textbf{30.73} {\scriptsize (-6.07)} &
\textbf{29.00} {\scriptsize (-7.80)} &
\textbf{32.76} {\scriptsize (-4.04)} &
\textbf{36.32} {\scriptsize (-0.48)} &
{60.06} {\scriptsize (+23.26)} \\

\midrule[1.5pt]
\multicolumn{3}{c|}{\textbf{Improvement Rate}} & \textbf{6/6} & \textbf{5/6} & \textbf{6/6} & \textbf{4/6} & {2/6} \\
\midrule[1.5pt]
\end{tabular}
\end{adjustbox}

\caption{Flatness metrics of ERM and different augmentations on (tiny)ImageNet datasets. Overall, augmentations with non-negligible density near the original image (i.e., better satisfying the PSA condition) tend to find flatter minima than ERM. This trend is clear but not perfectly monotone.}
\label{tab:flatness_metric_imagenet}
\end{table*}

\begin{table*}[t]
\centering
\begin{adjustbox}{width=0.96\textwidth}
\begin{tabular}{c|c|ccccc}
\midrule[1.5pt]
\textbf{$\quad \ $ Benchmarks $\quad \ $} & \ \ ERM \ \ & 
$\quad \ $ AugMix $\quad \ $ & $\quad \ $ DeepAug $\quad \ $ & $\quad \ $ PixMix $\quad \ $ & $\quad \ $ RandAug $\quad \ $ & $\quad \ $ StyleAug $\quad \ $ \\
\midrule[1.5pt]

\textit{tiny}ImageNet-C $\downarrow$
& 74.76
& \textbf{63.48} {\scriptsize (-11.28)}
& \textbf{56.14} {\scriptsize (-18.62)}
& \textbf{61.51} {\scriptsize (-13.25)}
& \textbf{66.92} {\scriptsize (-7.84)}
& {79.42} {\scriptsize (+4.66)} \\

ImageNet-$\text{C}$ $\downarrow$
& 69.37 &
\textbf{69.34} \scriptsize(-0.03) &
\textbf{56.68} \scriptsize(-12.69) &
\textbf{61.51} \scriptsize(-7.86) &
\textbf{66.88} \scriptsize(-2.49) &
{75.52} \scriptsize(+6.15) \\

\textit{tiny}ImageNet, $L_2$ $\downarrow$
& 57.77
& {61.84} {\scriptsize (+4.07)}
& \textbf{52.86} {\scriptsize (-4.91)}
& {58.38} {\scriptsize (+0.61)}
& {63.19} {\scriptsize (+5.42)}
& {83.13} {\scriptsize (+25.36)} \\

\textit{tiny}ImageNet, $L_\infty$ $\downarrow$
& 99.94
& \textbf{99.93} {\scriptsize (-0.01)}
& {99.95} {\scriptsize (+0.01)}
& {99.98} {\scriptsize (+0.04)}
& {99.98} {\scriptsize (+0.04)}
& {100.0} {\scriptsize (+0.06)} \\

ImageNet, $L_2$ $\downarrow$
& 76.16 &
\textbf{76.06} \scriptsize(-0.10) &
\textbf{75.37} \scriptsize(-0.79) &
\textbf{75.44} \scriptsize(-0.72) &
\textbf{76.09} \scriptsize(-0.07) &
{87.71} \scriptsize(+11.55) \\

ImageNet, $L_\infty$ $\downarrow$
& 99.70 &
\textbf{99.69} \scriptsize(-0.01) &
\textbf{99.62} \scriptsize(-0.08) &
\textbf{99.63} \scriptsize(-0.07) &
99.70 \scriptsize(+0.00) &
{99.92} \scriptsize(+0.22) \\

\midrule[1.5pt]
\multicolumn{2}{c|}{\textbf{$\qquad \quad$ Improvement Rate $\qquad \quad$}} & \textbf{5/6} & \textbf{5/6} & \textbf{4/6} & 3/6 & {0/6} \\
\midrule[1.5pt]
\end{tabular}
\end{adjustbox}

\caption{Comparison of various augmentation methods against ERM across multiple robustness scenarios, including common corruptions (\textit{tiny}ImageNet-C, ImageNet-C) and adversarial attacks ($L_2$, $L_{\infty}$). Essentially, the closer an augmentation aligns with the PSA condition, the higher its potential to ensure robustness when confronted with differing distribution shifts.}
\label{tab:main_table_imagenet}
\end{table*}

\section{Discussion and Future Work}
\label{sec:discussion}

In this section, we outline the strengths, limitations, and potential directions for future research stemming from our work. The primary contributions of our study are:

\setlist[itemize]{topsep=2pt, parsep=3pt, partopsep=1pt, itemsep=3pt, leftmargin=2em}
\begin{itemize}
    \item We \textit{first} offer a series of theories to explain how augmentation can improve robustness in \textit{flat minima} viewpoint, end-to-end fashion.
    \begin{itemize}[label=--,leftmargin=1em]
        \item While prior research has explored the relationship between augmentations and robustness, to the best of our knowledge, no existing work has utilized the concept of flat minima to establish this connection.
    \end{itemize}
    
    \item We \textit{first} establish a formal connection between the family of label-invariant augmentations and robustness under \textit{general} distribution shifts.
    \begin{itemize}[label=--,leftmargin=1em]
        \item Previous studies faced limitations such as focusing solely on: i) experimental analysis, ii) particular augmentation, or iii) specific distributional shifts.
        \item In contrast, our theories and analyses do not impose such restrictions on augmentations or benchmarks. We validate our findings through experiments across diverse robustness benchmarks.
    \end{itemize}
\end{itemize}

Despite these contributions, our work has certain limitations.
First, we do not provide the kind of extensive experimental validation across many benchmarks as done in technical reports. Second, our theorems do not directly cover label-manipulating augmentations such as Mixup~\cite{Mixup:2018} and CutMix~\cite{CutMix:2019}. Although our framework can be extended to incorporate label-mixing augmentation by setting $\delta$ toward the input mixture, this would require fixing a specific label-mixing rule and would restrict theory generality to the particular augmentation. Thus, we omit label-manipulating augmentations from our main analysis, leaving integration and further schemes to future work.

While not commonly done in augmentation literature, adversarial training (AT) can also be regarded as a specific type of augmentation and serve as a baseline~\cite{adv_as_baseline, adv_as_baseline2}. AT improves robustness to norm-bounded attacks but stays brittle under large distribution shifts (e.g., common corruptions). We hypothesize that such a phenomenon occurs since AT has limitations in promoting large $b$-flat minima due to diminutive perturbations. Therefore, combining AT with augmentations that generate samples in both short-distance and long-distance range (e.g., AugMix) helps counteract its vulnerability to large shifts, by effectively boosting $\gamma_A\gg0$ (Table~\ref{tab:fgsm_flatness},~\ref{tab:fgsm_robustness}).

\begin{table}[t]
\centering
\begin{adjustbox}{width=0.42\textwidth}
\begin{tabular}{c|c|cc}

\midrule[1.5pt]
\textbf{Dataset} & \textbf{Metrics} & $ \quad \ \ $AT$ \ \ \ $ & $ \ \ \ $AT + AugMix $ \ \ \ $ \\
\midrule[1.5pt]

\multirow{3}{*}{CIFAR-10} 
& $\mu_{\text{PAC-Bayes}}\downarrow$ & $ \ \ $ 158.15 & \textbf{113.78} \scriptsize (-44.37)\\

& LPF $\downarrow$ & $ \ \ $ 1.25 & \textbf{0.58} \scriptsize (-0.67) \\

& $\epsilon_{\text{sharp}}$ $\downarrow$ & $ \ \ $ 16.36 & \textbf{10.29} \scriptsize (-6.07) \\

\midrule

\multirow{3}{*}{CIFAR-100} &

$\mu_{\text{PAC-Bayes}}\downarrow$ & $ \ \ $ 207.90 & \textbf{161.89} \scriptsize (-46.01)  \\

& LPF $\downarrow$ & $ \ \ $ 2.95 & \textbf{2.11} \scriptsize (-0.84) \\

&$\epsilon_{\text{sharp}}\downarrow$& $ \ \ $ 20.81 & \textbf{10.34} \scriptsize (-10.47) \\

\midrule[1.5pt]
\end{tabular}
\end{adjustbox}
\caption{Flatness evaluations of AT (FGSM, $\epsilon=8/255$) and AT + AugMix on CIFAR-10/100.}
\label{tab:fgsm_flatness}
\end{table}

\begin{table}[t]
\centering
\begin{adjustbox}{width=0.40\textwidth}
\begin{tabular}{c|c|cc}

\midrule[1.5pt]
\textbf{Robustness}  &  \textbf{Benchmarks}  &  AT  &  AT + AugMix  \\
\midrule[1.5pt]

& CIFAR-10-C $\downarrow$
& 24.12 &
\textbf{15.39} \scriptsize (-8.73) \\

Common & CIFAR-10-$\overline{\text{C}}$ $\downarrow$
& 3.25 &
\textbf{1.30} \scriptsize (-1.95) \\

Corruption &  CIFAR-100-C $\downarrow$
& 52.50 &
\textbf{44.03} \scriptsize (-8.47) \\

 & CIFAR-100-$\overline{\text{C}}$ $\downarrow$
& 7.36 &
\textbf{3.74} \scriptsize (-3.62) \\

\midrule

& CIFAR-10, $L_2$ $\downarrow$
& 69.24 &
\textbf{62.86} \scriptsize (-6.38) \\

Adversarial & CIFAR-10, $L_\infty$ $\downarrow$
& 94.55 &
98.40 \scriptsize (+3.85) \\

Attacks & CIFAR-100, $L_2$ $\downarrow$
& 89.28 &
\textbf{88.60} \scriptsize (-0.68) \\

& CIFAR-100, $L_\infty$ $\downarrow$
& 98.42 &
99.75 \scriptsize (+1.33)\\
\midrule[1.5pt]
\end{tabular}
\end{adjustbox}
\caption{CIFAR-10/100 robustness benchmark results of AT and AT + AugMix.}
\label{tab:fgsm_robustness}
\end{table}

Finally, it is well known that models trained with flat-minima optimizers, such as Sharpness-Aware Minimization (SAM)~\cite{SAM:2021} and Stochastic Weight Averaging (SWA)~\cite{SWA:2018}, generalize better under distribution shifts. However, the mechanism by which data augmentation fosters flat minima remains underexplored. We bridge this gap by introducing a simple sufficient "proximal support" condition and proving that PSA encourages flat minima. See Appendix D for further discussions on applications, flat-minima optimizers, and adversarial training.

\section{Related Works}
\label{sec:related_work}

\subsection{Data Augmentations}
Existing augmentation techniques can be categorized into model-free, model-based, and policy-based approaches, following the classification proposed in \citet{augment_survey}.

\textit{Model-free augmentations} are further subdivided into single-image and multi-image techniques. Single-image methods encompass basic image operations like translation, rotation, and color jitter, as well as masking strategies such as CutOut \cite{CutOut} and Hide-and-Seek \cite{HideNSeek}. These transformations have been shown to enhance model performance on target data distributions efficiently and with low overhead. Multi-image augmentations involve composing multiple operations and blending elements from different images. Among label-manipulating multi-image methods are techniques that combine pairs of distinct images and their labels, including Mixup \cite{Mixup:2018} and CutMix \cite{CutMix:2019}. These approaches are widely recognized for their ability to boost model performance. Subsequently, label-preserving multi-image methods like AugMix \cite{AugMix9:2021} and PixMix \cite{PixMix:2022} have demonstrated superior effectiveness in enhancing model robustness.

\textit{Model-based augmentations} leverage pretrained models to produce augmented data. Several techniques employing generative models, such as CGAN \cite{CGAN} and its variants \cite{CGAN_var1, CGAN_var2, CGAN_var3, CGAN_var4}, are designed to mitigate data imbalance issues. Additional methods, including DeepAugment \cite{DeepAug10:2021}, ANT \cite{ANT11:2020}, and StyleAug \cite{StyleAug}, focus on improving classifier robustness to common corruptions, adversarial attacks, or domain shifts. More recently, pretrained diffusion models combined with effective prompt engineering \cite{Diffusion1, Diffusion2} has proven successful in enhancing performance across various vision tasks.

\textit{Policy-based augmentations} focus on designing an automatic way to determine the optimal augmentation strategies by employing reinforcement learning or adversarial training. From a pioneering method, AutoAugment \cite{AutoAugment} that utilizes reinforcement learning for finding the best augmentation strategies, 
subsequent works such as Fast AA \cite{FastAA}, Faster AA \cite{FasterAA}, and RandAugment \cite{RandAugment} aim to enhance both the efficiency of policy search and the model performance. 
Adversarial training-based augmentation strategies, including AdaTransform \cite{AdaTransform}, Adversarial AA \cite{AdvAA}, and AugMax \cite{AugMax} leverage adversarial perturbations which maximally disturbs samples to be misclassified into other labels, finally leading to improve model robustness against unseen domains.

Despite the effectiveness of augmentation methods in enhancing model performance, prior studies have focused more on their practical use rather than on understanding their theoretical impact on model robustness to data shifts.

\subsection{Augmentations and Model Robustness}

A number of previous works have tried to reveal the relationship between augmentation and model robustness. 

\citet{aug_analysis_cc1} has provided a game-theoretic perspective on augmentations and introduces a new proxy metric for measuring common corruption robustness.
\citet{ME-ADA} and \citet{ADA} have theoretically found out that adversarial perturbations in the latent space can simulate worst-case distributional shifts in the data. \citet{AugmentNAdvGood} have empirically found out that Mixup and CutOut with model weight averaging is shown to improve adversarial robustness. 
\citet{boundary_thickness} has shown that Mixup enlarges boundary thickness, the marginal space between differently labeled samples, which is deemed to be highly correlated with adversarial robustness and common corruptions.
\citet{AugFourier} interpret the augmentation-originated gains in model robustness by explaining that augmentations make deep models utilize both high and low frequency information of images so as to enhance model robustness against data corruptions.
\citet{NoLabelNAdvGood1} and \citet{NoLabelNAdvGood2} have theoretically shown that utilizing unlabeled data in training can improve adversarial robustness. 
\citet{DeepAug10:2021} have empirically found out that exploiting diverse augmentations together improves robustness against both adversarial and common corruptions.

Nonetheless, most of the prior interpretations of how augmentation contributes to model robustness are either confined to specific types of augmentations and robustness or empirical analysis. 
None of the prior works generally explain how augmentations can theoretically improve model robustness across diverse distributional shifts.

\subsection{Flat Minima and Robustness}

The relationship between flat minima in the loss landscape and improved generalization performance has been well-established in several seminal works~\cite{PAC_DG1,SWAD:2021,zhang2024duality,CL_flat,FL_flat}. Flat minima refer to regions in the optimization landscape where the loss function remains relatively stable and insensitive to small perturbations in the model parameters, which often correlate with better generalization to unseen data. For instance, \citet{PAC_DG1} establishes a rigorous theoretical link between these flat minima and superior generalization capabilities in over-parameterized neural networks, achieving this by refining the PAC-Bayes generalization bound originally proposed in~\citet{PAC_DG2}. This refinement provides a more precise quantification of how flatter regions can lead to models that perform more reliably on out-of-distribution data.

In a similar vein, \citet{SWAD:2021} demonstrates that flat minima contribute to tighter bounds on generalization error, supported by both theoretical analysis and practical experiments. Their approach highlights dense weight averaging as a practical and efficient heuristic for locating such minima during training, showing through empirical validation that it consistently enhances model performance across various benchmarks. Furthermore, \citet{zhang2024duality} investigates how sharpness-aware minimization (SAM) relates to adversarial training (AT), demonstrating that SAM can replicate AT’s behavior under the restrictive setting of Gaussian inputs and linear models. Because their theoretical analysis is confined to binary classification and Gaussian data, our theorems and the PSA condition have an edge in applying to general data distributions.

Despite these advances, prior work still leaves the open question of how \textit{data augmentation} itself steers optimization toward flat minima. Our contribution addresses this critical gap by forging a direct connection between data augmentation strategies and the identification of flat minima, while also demonstrating their pivotal role in bolstering model robustness against a wide array of general distribution shifts, such as those encountered in real-world scenarios involving domain variations or noisy inputs.

\subsection{Other Related Works}

\citet{jiang2023chasing} investigates how to preserve fairness under distribution shifts by reinterpreting such shifts as equivalent perturbations to both model weights and inputs, employing a transportation function. Still, their work focuses on group fairness with respect to demographic parity and equal opportunity, and demonstrates only the feasibility of this equivalence without further extensions.

\citet{flat_minima_cc} shows in binary classification that flatter minima mitigate the boundary tilting problem~\cite{boundary_tilting}, and empirically finds these minima correlate with greater robustness to common corruptions. However, the theory remains limited, as the boundary tilting problem is closely tied to adversarial robustness and is not explored beyond binary classification.

\section{Conclusion}\label{sec:conclusion}
Our framework connects data augmentation to robustness through flat minima viewpoint under data shifts. Augmentations meeting the PSA condition form broad, flat regions in parameter space, while others provide limited protection. Experiments on CIFAR and ImageNet substantiate our theorems, and we expect our work to inspire next-generation augmentation methods, including foundation model-based.

\section{Acknowledgments}
This work was supported by the National Research Foundation of Korea (NRF) grant funded by
the Korea government (MSIT) (No. RS-2024-00459023), Institute of Information \& Communications Technology Planning \& Evaluation (IITP) grant funded by the Korea government (MSIT) (No. RS-2020-II201336, Artificial Intelligence Graduate School Program at UNIST), (No. RS-2025-25442824, AI Star Fellowship Program at UNIST), (No. RS-2025-25441996, Development of a Virtual Tactile Signal Generation Platform Technology Based on Multimodal Vision-Tactile Integrated AGI), (No. IITP-2025-RS-2022-00156361, Innovative Human Resource Development for Local Intellectualization program).

\nocite{lazytraining0,lazytraining1,lazytraining2,lazytraining4,NTK0,NTK1,Tame1,Tame2,jacobianfullrank0,jacobianfullrank1, boundary_thickness}
\bibliography{aaai2026}

@String(CVPR  = {Conference on Computer Vision and Pattern Recognition (CVPR)})

@String(ICCV  = {ICCV})

@String(ECCV  = {European Conference on Computer Vision (ECCV)})

@String(NIPS  = {Neural Information Processing Systems (NeurIPS)})

@String(ICPR  = {ICPR})

@String(ICASSP=	{ICASSP})

@String(ICLR  = {International Conference on Learning Representations (ICLR)})

@String(ICML = {ICML})

@article{zhang2024duality,
  title={On the duality between sharpness-aware minimization and adversarial training},
  author={Zhang, Yihao and He, Hangzhou and Zhu, Jingyu and Chen, Huanran and Wang, Yifei and Wei, Zeming},
  journal={arXiv preprint arXiv:2402.15152},
  year={2024}
}

@inproceedings{CL_flat,
 author = {Bian, Ang and Li, Wei and Yuan, Hangjie and Yu, Chengrong and Wang, Mang and Zhao, Zixiang and Lu, Aojun and Ji, Pengliang and Feng, Tao},
 booktitle = {Advances in Neural Information Processing Systems (NeurIPS)},
 editor = {A. Globerson and L. Mackey and D. Belgrave and A. Fan and U. Paquet and J. Tomczak and C. Zhang},
 pages = {7608--7630},
 publisher = {Curran Associates, Inc.},
 title = {Make Continual Learning Stronger via C-Flat},
 url = {https://proceedings.neurips.cc/paper_files/paper/2024/file/0e705ac30e573d1526f81a0fd071a151-Paper-Conference.pdf},
 volume = {37},
 year = {2024}
}

@inproceedings{FL_flat,
  title={Rethinking the flat minima searching in federated learning},
  author={Lee, Taehwan and Yoon, Sung Whan},
  booktitle={Forty-first International Conference on Machine Learning (ICML)},
  year={2024}
}

@inproceedings{aug_analysis_cc1,
  title={Understanding Data Augmentation From A Robustness Perspective},
  author={Liu, Zhendong and Zhang, Jie and He, Qiangqiang and Wang, Chongjun},
  booktitle={ICASSP 2024-2024 IEEE International Conference on Acoustics, Speech and Signal Processing (ICASSP)},
  pages={6760--6764},
  year={2024},
  organization={IEEE}
}

@article{augment_survey,
    author = {Xu, Mingle and Yoon, Sook and Fuentes, Alvaro and Park, Dong Sun},
    title = {A Comprehensive Survey of Image Augmentation Techniques for Deep Learning},
    year = {2023},
    issue_date = {May 2023},
    publisher = {Elsevier Science Inc.},
    address = {USA},
    volume = {137},
    number = {C},
    issn = {0031-3203},
    url = {https://doi.org/10.1016/j.patcog.2023.109347},
    doi = {10.1016/j.patcog.2023.109347},
    journal = {Pattern Recogn.},
    month = {May},
    numpages = {12}
}

@inproceedings{flat_minima_cc,
  title={Flatter Minima of Loss Landscapes Correspond with Strong Corruption Robustness},
  author={Zhong, Liqun and Zhu, Kaijie and Yang, Ge},
  booktitle={International Conference on Pattern Recognition (ICPR)},
  pages={314--328},
  year={2024},
  organization={Springer}
}

@article{CGAN,
  title={Conditional generative adversarial nets},
  author={Mirza, Mehdi and Osindero, Simon},
  journal={arXiv preprint arXiv:1411.1784},
  year={2014}
}

@article{CGAN_var1,
  title={Effective data generation for imbalanced learning using conditional generative adversarial networks},
  author={Douzas, Georgios and Bacao, Fernando},
  journal={Expert Systems with applications},
  volume={91},
  pages={464--471},
  year={2018},
  publisher={Elsevier}
}

@article{CGAN_var2,
  title={Bagan: Data augmentation with balancing gan},
  author={Mariani, Giovanni and Scheidegger, Florian and Istrate, Roxana and Bekas, Costas and Malossi, Cristiano},
  journal={arXiv preprint arXiv:1803.09655},
  year={2018}
}

@article{CGAN_var3,
  title={MFC-GAN: Class-imbalanced dataset classification using multiple fake class generative adversarial network},
  author={Ali-Gombe, Adamu and Elyan, Eyad},
  journal={Neurocomputing},
  volume={361},
  pages={212--221},
  year={2019},
  publisher={Elsevier}
}

@inproceedings{CGAN_var4,
  title={Ida-gan: A novel imbalanced data augmentation gan},
  author={Yang, Hao and Zhou, Yun},
  booktitle={2020 25th International Conference on Pattern Recognition (ICPR)},
  pages={8299--8305},
  year={2021},
  organization={IEEE}
}

@inproceedings{AutoAugment,
  title={Autoaugment: Learning augmentation strategies from data},
  author={Cubuk, Ekin D and Zoph, Barret and Mane, Dandelion and Vasudevan, Vijay and Le, Quoc V},
  booktitle=CVPR,
  pages={113--123},
  year={2019}
}

@article{FastAA,
  title={Fast autoaugment},
  author={Lim, Sungbin and Kim, Ildoo and Kim, Taesup and Kim, Chiheon and Kim, Sungwoong},
  journal=NIPS,
  volume={32},
  year={2019}
}

@inproceedings{FasterAA,
  title={Faster autoaugment: Learning augmentation strategies using backpropagation},
  author={Hataya, Ryuichiro and Zdenek, Jan and Yoshizoe, Kazuki and Nakayama, Hideki},
  booktitle=ECCV,
  pages={1--16},
  year={2020},
  organization={Springer}
}

@inproceedings{RandAugment,
  title={Randaugment: Practical automated data augmentation with a reduced search space},
  author={Cubuk, Ekin D and Zoph, Barret and Shlens, Jonathon and Le, Quoc V},
  booktitle=CVPR,
  pages={702--703},
  year={2020}
}

@inproceedings{AdaTransform,
  title={Adatransform: Adaptive data transformation},
  author={Tang, Zhiqiang and Peng, Xi and Li, Tingfeng and Zhu, Yizhe and Metaxas, Dimitris N},
  booktitle={Proceedings of the IEEE/CVF International Conference on Computer Vision (ICCV)},
  pages={2998--3006},
  year={2019}
}

@inproceedings{
AdvAA,
title={Adversarial AutoAugment},
author={Xinyu Zhang and Qiang Wang and Jian Zhang and Zhao Zhong},
booktitle=ICLR,
year={2020},
url={https://openreview.net/forum?id=ByxdUySKvS}
}

@inproceedings{AugMax,
  title={AugMax: Adversarial Composition of Random Augmentations for Robust Training},
  author={Wang, Haotao and Xiao, Chaowei and Kossaifi, Jean and Yu, Zhiding and Anandkumar, Anima and Wang, Zhangyang},
  booktitle=NIPS,
  year={2021}
}

@article{ME-ADA,
  title={Maximum-entropy adversarial data augmentation for improved generalization and robustness},
  author={Zhao, Long and Liu, Ting and Peng, Xi and Metaxas, Dimitris},
  journal=NIPS,
  volume={33},
  pages={14435--14447},
  year={2020}
}

@article{ADA,
  title={Generalizing to unseen domains via adversarial data augmentation},
  author={Volpi, Riccardo and Namkoong, Hongseok and Sener, Ozan and Duchi, John C and Murino, Vittorio and Savarese, Silvio},
  journal=NIPS,
  volume={31},
  year={2018}
}

@article{PixMix:2022,
  title={PixMix: Dreamlike Pictures Comprehensively Improve Safety Measures},
  author={Dan Hendrycks and Andy Zou and Mantas Mazeika and Leonard Tang and Bo Li and Dawn Song and Jacob Steinhardt},
  journal={CVPR},
  year={2022}
}

@article{CutOut,  
  title={Improved Regularization of Convolutional Neural Networks with Cutout},  
  author={DeVries, Terrance and Taylor, Graham W},  
  journal={arXiv preprint arXiv:1708.04552},  
  year={2017}  
}

@InProceedings{HideNSeek,
author = {Kumar Singh, Krishna and Jae Lee, Yong},
title = {Hide-And-Seek: Forcing a Network to Be Meticulous for Weakly-Supervised Object and Action Localization},
booktitle = ICCV,
month = {Oct},
year = {2017}
}

@inproceedings{AugmentNAdvGood,
 author = {Rebuffi, Sylvestre-Alvise and Gowal, Sven and Calian, Dan Andrei and Stimberg, Florian and Wiles, Olivia and Mann, Timothy A},
 booktitle = NIPS,
 editor = {M. Ranzato and A. Beygelzimer and Y. Dauphin and P.S. Liang and J. Wortman Vaughan},
 pages = {29935--29948},
 publisher = {Curran Associates, Inc.},
 title = {Data Augmentation Can Improve Robustness},
 volume = {34},
 year = {2021}
}

@inproceedings{NoLabelNAdvGood1,
 author = {Najafi, Amir and Maeda, Shin-ichi and Koyama, Masanori and Miyato, Takeru},
 booktitle = NIPS,
 editor = {H. Wallach and H. Larochelle and A. Beygelzimer and F. d\textquotesingle Alch\'{e}-Buc and E. Fox and R. Garnett},
 publisher = {Curran Associates, Inc.},
 title = {Robustness to Adversarial Perturbations in Learning from Incomplete Data},
 volume = {32},
 year = {2019}
}

@inproceedings{NoLabelNAdvGood2,
 author = {Alayrac, Jean-Baptiste and Uesato, Jonathan and Huang, Po-Sen and Fawzi, Alhussein and Stanforth, Robert and Kohli, Pushmeet},
 booktitle = NIPS,
 editor = {H. Wallach and H. Larochelle and A. Beygelzimer and F. d\textquotesingle Alch\'{e}-Buc and E. Fox and R. Garnett},
 publisher = {Curran Associates, Inc.},
 title = {Are Labels Required for Improving Adversarial Robustness?},
 volume = {32},
 year = {2019}
}

@inproceedings{AugFourier,
 author = {Yin, Dong and Gontijo Lopes, Raphael and Shlens, Jon and Cubuk, Ekin Dogus and Gilmer, Justin},
 booktitle = NIPS,
 editor = {H. Wallach and H. Larochelle and A. Beygelzimer and F. d\textquotesingle Alch\'{e}-Buc and E. Fox and R. Garnett},
 publisher = {Curran Associates, Inc.},
 title = {A Fourier Perspective on Model Robustness in Computer Vision},
 volume = {32},
 year = {2019}
}

@inproceedings{IN-C7:2019,
  title={Benchmarking neural network robustness to common corruptions and perturbations},
  author={Dan Hendrycks and Thomas Dietterich},
  booktitle={International Conference on Learning Representations (ICLR)},
  year={2019}
}

@inproceedings{AugMix9:2021,
  title={Augmix: A simple data processing method to improve robustness and uncertainty},
  author={Dan Hendrycks and Norman Mu and Ekin D. Cubuk and Barret Zoph and Justin Gilmer and Balaji Lakshminarayanan},
  booktitle={International Conference on Learning Representations (ICLR)},
  year={2021}
}

@inproceedings{DeepAug10:2021,
  title={The many faces of robustness: A critical analysis of out-of-distribution generalization},
  author={Dan Hendrycks and Steven Basart and Norman Mu and Saurav Kadavath and Frank Wang and Evan Dorundo and Rahul Desai and Tyler Zhu and Samyak Parajuli and Mike Guo and Dawn Song and Jacob Steinhardt and Justin Gilmer},
  booktitle={International Conference on Computer Vision (ICCV)},
  year={2021}
}

@inproceedings{ANT11:2020,
  title={A simple way to make neural networks robust against diverse image corruptions},
  author={Evgenia Rusak and Lukas Schott and Roland S. Zimmermann and Julian Bitterwolf and Oliver Bringmann and Matthias Bethge and Wieland Brendel},
  booktitle={European Conference on Computer Vision (ECCV)},
  year={2020}
}

@article{boundary_thickness,
  title={Boundary thickness and robustness in learning models},
  author={Yang, Yaoqing and Khanna, Rajiv and Yu, Yaodong and Gholami, Amir and Keutzer, Kurt and Gonzalez, Joseph E and Ramchandran, Kannan and Mahoney, Michael W},
  journal={Advances in Neural Information Processing Systems (NeurIPS)},
  volume={33},
  year={2020}
}

@inproceedings{Mixup:2018,
  title={Mixup: Beyond Empirical Risk Minimization},
  author={Honyi Zhang and Moustapha Cisse and Yann N. Dauphin and David Lopez-Paz},
  booktitle={International Conference on Machine Learning (ICML)},
  year={2018}
}

@inproceedings{CutMix:2019,
    title={CutMix: Regularization Strategy to Train Strong Classifiers with Localizable Features},
    author={Yun, Sangdoo and Han, Dongyoon and Oh, Seong Joon and Chun, Sanghyuk and Choe, Junsuk and Yoo, Youngjoon},
    booktitle = {International Conference on Computer Vision (ICCV)},
    year={2019},
    pubstate={published},
    tppubtype={inproceedings}
}

@inproceedings{SWA:2018,
  title={Averaging weights leads to wider optima and better generalization.},
  author={Pavel Izmailov and Dmitrii Podoprikhin and Timur Garipov and Dmitry Vetrov and Andrew Gordon},
  booktitle={Conference on Uncertainty in Artificial Intelligence(UAI)},
  year={2018}
}

@inproceedings{SWAD:2021,
  title={SWAD: Domain Generalization by Seeking Flat Minima},
  author={Junbum Cha and Sanghyuk Chun and Kyungjae Lee and Han-Cheol Cho and Seunghyun Park and Yunsung Lee and Sungrae Park},
  booktitle={Advances in Neural Information Processing Systems (NeurIPS)},
  year={2021}
}

@inproceedings{IN-C-bar:2021,
  title={On Interaction Between Augmentations and Corruptions in Natural Corruption Robustness},
  author={Eric Mintun and Alexander Kirillov and Saining Xie},
  booktitle={Advances in Neural Information Processing Systems (NeurIPS)},
  year={2021}
}

@inproceedings{IN,
  title={ImageNet: A Large-Scale Hierarchical Image Database},
  author={Deng, Jia and Dong, Wei and Socher, Richard and Li, Li-Jia and Li, Kai and Fei-Fei, Li},
  booktitle={2009 IEEE Conference on Computer Vision and Pattern Recognition},
  pages={248--255},
  year={2009},
  organization={IEEE}
}

@inproceedings{SAM:2021,
title={Sharpness-aware Minimization for Efficiently Improving Generalization},
author={Pierre Foret and Ariel Kleiner and Hossein Mobahi and Behnam Neyshabur},
booktitle={International Conference on Learning Representations (ICLR)},
year={2021},
url={https://openreview.net/forum?id=6Tm1mposlrM}
}

@inproceedings{PGD,
title={Towards Deep Learning Models Resistant to Adversarial Attacks},
author={Aleksander Madry and Aleksandar Makelov and Ludwig Schmidt and Dimitris Tsipras and Adrian Vladu},
booktitle={International Conference on Learning Representations (ICLR)},
year={2018},
url={https://openreview.net/forum?id=rJzIBfZAb},
}

@inproceedings{StyleAug,
  title={Style augmentation: data augmentation via style randomization.},
  author={Jackson, Philip TG and Abarghouei, Amir Atapour and Bonner, Stephen and Breckon, Toby P and Obara, Boguslaw},
  booktitle={CVPR workshops},
  volume={6},
  pages={10--11},
  year={2019}
}

@inproceedings{boundary_tilting,
  title={A Boundary Tilting Persepective on the Phenomenon of Adversarial Examples},
  author={Thomas Tanay and Lewis Griffin},
  booktitle={	arXiv:1608.07690},
  year={2016}
}

@article{shi2021overcoming,
  title={Overcoming catastrophic forgetting in incremental few-shot learning by finding flat minima},
  author={Shi, Guangyuan and Chen, Jiaxin and Zhang, Wenlong and Zhan, Li-Ming and Wu, Xiao-Ming},
  journal={Advances in Neural Information Processing Systems (NeurIPS)},
  volume={34},
  pages={6747--6761},
  year={2021}
}

@inproceedings{LPF,
  title={Low-pass filtering sgd for recovering flat optima in the deep learning optimization landscape},
  author={Bisla, Devansh and Wang, Jing and Choromanska, Anna},
  booktitle={International Conference on Artificial Intelligence and Statistics (AISTATS)},
  pages={8299--8339},
  year={2022},
  organization={PMLR}
}

@inproceedings{eps-sharpness,
title={On Large-Batch Training for Deep Learning: Generalization Gap and Sharp Minima},
author={Nitish Shirish Keskar and Dheevatsa Mudigere and Jorge Nocedal and Mikhail Smelyanskiy and Ping Tak Peter Tang},
booktitle={International Conference on Learning Representations (ICLR)},
year={2017},
url={https://openreview.net/forum?id=H1oyRlYgg}
}

@inproceedings{PACBayes,
  title={Fantastic generalization measures and
where to find them},
  author={Yiding Jiang and Behnam Neyshabur and Hossein Mobahi and Dilip Krishnan and Samy Bengio},
  booktitle={International Conference on Learning Representations (ICLR)},
  year={2020}
}

@inproceedings{DomainBed,
title={In Search of Lost Domain Generalization},
author={Ishaan Gulrajani and David Lopez-Paz},
booktitle={International Conference on Learning Representations (ICLR)},
year={2021},
url={https://openreview.net/forum?id=lQdXeXDoWtI}
}

@article{jiang2023chasing,
  title={Chasing fairness under distribution shift: a model weight perturbation approach},
  author={Jiang, Zhimeng Stephen and Han, Xiaotian and Jin, Hongye and Wang, Guanchu and Chen, Rui and Zou, Na and Hu, Xia},
  journal={Advances in Neural Information Processing Systems (NeurIPS)},
  volume={36},
  year={2023}
}

@article{lazytraining0,
  title={On lazy training in differentiable programming},
  author={Chizat, Lenaic and Oyallon, Edouard and Bach, Francis},
  journal={Advances in Neural Information Processing Systems (NeurIPS)},
  volume={32},
  year={2019}
}

@inproceedings{lazytraining1,
  title={Gradient Descent Provably Optimizes Over-parameterized Neural Networks},
  author={Du, Simon S and Zhai, Xiyu and Poczos, Barnabas and Singh, Aarti},
  booktitle={International Conference on Learning Representations (ICLR)},
  year={2018}
}

@article{lazytraining2,
  title={Learning overparameterized neural networks via stochastic gradient descent on structured data},
  author={Li, Yuanzhi and Liang, Yingyu},
  journal={Advances in Neural Information Processing Systems (NeurIPS)},
  volume={31},
  year={2018}
}

@inproceedings{lazytraining4,
  title={A convergence theory for deep learning via over-parameterization},
  author={Allen-Zhu, Zeyuan and Li, Yuanzhi and Song, Zhao},
  booktitle={International conference on machine learning (ICML)},
  pages={242--252},
  year={2019},
  organization={PMLR}
}

@article{NTK0,
  title={Neural tangent kernel: Convergence and generalization in neural networks},
  author={Jacot, Arthur and Gabriel, Franck and Hongler, Cl{\'e}ment},
  journal={Advances in Neural Information Processing Systems (NeurIPS)},
  volume={31},
  year={2018}
}

@article{NTK1,
  title={Wide neural networks of any depth evolve as linear models under gradient descent},
  author={Lee, Jaehoon and Xiao, Lechao and Schoenholz, Samuel and Bahri, Yasaman and Novak, Roman and Sohl-Dickstein, Jascha and Pennington, Jeffrey},
  journal={Advances in Neural Information Processing Systems (NeurIPS)},
  volume={32},
  year={2019}
}

@article{Sard,
  title={The measure of the critical values of differentiable maps},
  author={Sard, Arthur},
  journal={Bulletin of the American Mathematical Society},
  year={1942}
}

@article{jacobianfullrank0,
  title={Fast convergence of natural gradient descent for over-parameterized neural networks},
  author={Zhang, Guodong and Martens, James and Grosse, Roger B},
  journal={Advances in Neural Information Processing Systems (NeurIPS)},
  volume={32},
  year={2019}
}

@article{jacobianfullrank1,
  title={Rank diminishing in deep neural networks},
  author={Feng, Ruili and Zheng, Kecheng and Huang, Yukun and Zhao, Deli and Jordan, Michael and Zha, Zheng-Jun},
  journal={Advances in Neural Information Processing Systems (NeurIPS)},
  volume={35},
  pages={33054--33065},
  year={2022}
}

@inproceedings{aug_analysis_custom1,
  title={Fuzz testing based data augmentation to improve robustness of deep neural networks},
  author={Gao, Xiang and Saha, Ripon K and Prasad, Mukul R and Roychoudhury, Abhik},
  booktitle={Proceedings of the acm/ieee 42nd international conference on software engineering},
  pages={1147--1158},
  year={2020}
}

@article{aug_analysis_adv1,
  title={Fixing data augmentation to improve adversarial robustness},
  author={Rebuffi, Sylvestre-Alvise and Gowal, Sven and Calian, Dan A and Stimberg, Florian and Wiles, Olivia and Mann, Timothy},
  journal={arXiv preprint arXiv:2103.01946},
  year={2021}
}

@InProceedings{aug_analysis_dg,
    author    = {Li, Pan and Li, Da and Li, Wei and Gong, Shaogang and Fu, Yanwei and Hospedales, Timothy M.},
    title     = {A Simple Feature Augmentation for Domain Generalization},
    booktitle = {Proceedings of the IEEE/CVF International Conference on Computer Vision (ICCV)},
    month     = {October},
    year      = {2021},
    pages     = {8886-8895}
}

@article{CIFAR,
  title={Learning multiple layers of features from tiny images},
  author={Krizhevsky, Alex and Hinton, Geoffrey and others},
  year={2009},
  journal={Toronto, ON, Canada}
}

@inproceedings{allconvnet,
  title={Striving for Simplicity: The All Convolutional Net},
  author={Springenberg, J and Dosovitskiy, Alexey and Brox, Thomas and Riedmiller, M},
  booktitle={International Conference on Learning Representations (ICLR) workshop track},
  year={2015}
}

@inproceedings{densenet,
  title={Densely connected convolutional networks},
  author={Huang, Gao and Liu, Zhuang and Van Der Maaten, Laurens and Weinberger, Kilian Q},
  booktitle={Proceedings of the IEEE conference on computer vision and pattern recognition (CVPR)},
  pages={4700--4708},
  year={2017}
}

@inproceedings{resnext,
  title={Aggregated residual transformations for deep neural networks},
  author={Xie, Saining and Girshick, Ross and Doll{\'a}r, Piotr and Tu, Zhuowen and He, Kaiming},
  booktitle={Proceedings of the IEEE conference on computer vision and pattern recognition (CVPR)},
  pages={1492--1500},
  year={2017}
}

@inproceedings{PAC_DG1,
  title={A PAC-Bayesian Link Between Generalisation and Flat Minima},
  author={Haddouche, Maxime and Viallard, Paul and {\c{S}}im{\c{s}}ekli, Umut and Guedj, Benjamin},
  booktitle={ALT 2025-36th International Conference on Algorithmic Learning Theory},
  pages={1--31},
  year={2025}
}

@article{PAC_DG2,
  title={Computing nonvacuous generalization bounds for deep (stochastic) neural networks with many more parameters than training data},
  author={Dziugaite, Gintare Karolina and Roy, Daniel M},
  journal={arXiv preprint arXiv:1703.11008},
  year={2017}
}

@inproceedings{Diffusion1,
  title={Diffusemix: Label-preserving data augmentation with diffusion models},
  author={Islam, Khawar and Zaheer, Muhammad Zaigham and Mahmood, Arif and Nandakumar, Karthik},
  booktitle={Proceedings of the IEEE/CVF Conference on Computer Vision and Pattern Recognition (CVPR)},
  pages={27621--27630},
  year={2024}
}

@inproceedings{Diffusion2,
  title={A simple background augmentation method for object detection with diffusion model},
  author={Li, Yuhang and Dong, Xin and Chen, Chen and Zhuang, Weiming and Lyu, Lingjuan},
  booktitle={European Conference on Computer Vision (ECCV)},
  pages={462--479},
  year={2024},
  organization={Springer}
}

@article{Tame1,
  title={An inertial Newton algorithm for deep learning},
  author={Castera, Camille and Bolte, J{\'e}r{\^o}me and F{\'e}votte, C{\'e}dric and Pauwels, Edouard},
  journal={Journal of Machine Learning Research (JMLR)},
  volume={22},
  number={134},
  pages={1--31},
  year={2021}
}

@article{Tame2,
  title={Tame functions are semismooth},
  author={Bolte, J{\'e}r{\^o}me and Daniilidis, Aris and Lewis, Adrian},
  journal={Mathematical Programming},
  volume={117},
  number={1},
  pages={5--19},
  year={2009},
  publisher={Springer}
}

@article{MorseSard,
  title={The behavior of a function on its critical set},
  author={Morse, Anthony P},
  journal={Annals of Mathematics},
  volume={40},
  number={1},
  pages={62--70},
  year={1939},
  publisher={JSTOR}
}

@article{adv_as_baseline,
  title={Efficient and effective augmentation strategy for adversarial training},
  author={Addepalli, Sravanti and Jain, Samyak and others},
  journal={Advances in Neural Information Processing Systems (NeurIPS)},
  volume={35},
  pages={1488--1501},
  year={2022}
}

@inproceedings{adv_as_baseline2,
  title={Data Augmentation Alone Can Improve Adversarial Training},
  author={Li, Lin and Spratling, Michael},
  booktitle={The Eleventh International Conference on Learning Representations (ICLR)},
  year={2023}

}

%%%%%%%%%%%%%%%%%%%%%%%%%%%%%%%%%%%%%%%%%%%%%%%%%%%%%%%%%%%%%%%%%%%%

\newpage

\appendix

\onecolumn

\section*{A. Supplementary Materials}

This appendix contains additional material that could not be incorporated into the main paper due to page constraints, including detailed proofs of the theorems and experimental details. Section B provides detailed proofs of Theorems 1 to 3, Section C presents experimental details and additional results, and Section D offers further discussion on broader applicability of our findings including practical usage.

\setcounter{theorem}{1}
\setcounter{equation}{5}

\section*{B. Proofs on Theorems}
\label{proof_apdx}
\subsection{B.1. Proof on Theorem 1}
\label{proof_thm1}

\begin{figure*}[ht]
	\begin{center}
		\includegraphics[width=1.0\linewidth]{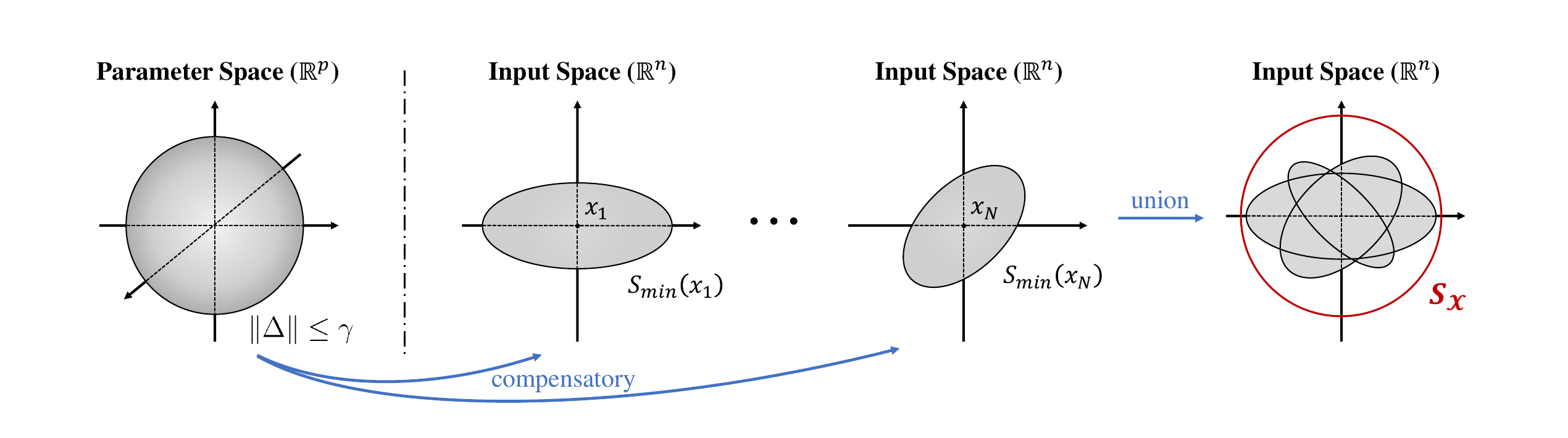}
	\end{center}
	\caption{An illustration of Theorem 1 (A). Given a parameter perturbation set $\{\Delta \in \mathbb{R}^p : \| \Delta \| \leq \gamma\}$ and datapoints $x_1, ..., x_n$, each $x_i$ has a minimal functional compensatory region $\mathcal{S}_{min}(x_i)$ shaped like an ellipsoid. All of the ellipsoids can be collectively enclosed by an isotropic ball $\mathcal{S}_\mathcal{X}$, whose radius is determined by the largest principal axis across the ellipsoids. The exact reverse process is done for the Theorem 1 (B).}
\label{fig:dummy}
\end{figure*}

\textbf{Theorem 1 (A).}
\textit{
(Parameter-to-Input) 
Consider a function $f: \mathbb{R}^n \times \mathbb{R}^p \rightarrow \mathbb{R}^c$. For any $\gamma>0$, define the ball of parameter perturbations $\{\Delta \in \mathbb{R}^p : \|\Delta\|\le \gamma\}$ and 
$$
\mathcal{S}_\mathcal{X} \;=\; \left\{\delta \in \mathbb{R}^n : \|\delta\|\;\le\;\max_{x \in \mathcal{D}}\!\Bigl(\dfrac{\sigma_{\max}^\theta}{\sigma_{\min}^{x}}\Bigr)\,\gamma \right\}.
$$
Then, $\mathcal{S}_\mathcal{X}$ is always a functional compensatory set for the parameter perturbation ball and dataset $\mathcal{D}$.
}

\begin{proof} As a brief recap regarding notations on Jacobian matrices, $\mathbf{J}$ is the Jacobian of a function $f: \mathbb{R}^n \times \mathbb{R}^ p \rightarrow \mathbb{R}^c$, and the Jacobian matrix $\mathbf{J}$ can be expressed as its input and parameter side differentials, i.e. $\mathbf{J} = \left[ \mathbf{J_x} \ \ \mathbf{J}_{\boldsymbol{\theta}} \right]$. The notation $\mathbf{J_x}(x_i)$ and $\mathbf{J}_{\boldsymbol{\theta}}(\theta)$ will be occasionally used to represent the evaluation of $\mathbf{J_x}$ and $\mathbf{J}_{\boldsymbol{\theta}}$ at point $x_i \in \mathbb{R}^n$ and $\theta \in \mathbb{R}^p$.

Using the first-order Taylor expansion, $f(x; \theta + \Delta) \approx f(x ; \theta) + \mathbf{J}_{\boldsymbol{\theta}} \cdot \Delta$. This approximation is supported by two main observations. First, when $f$ is a linear model, the Taylor expansion becomes exact, and a linear model with infinite width serves as a universal approximator. Second, supported by lazy training regime~\cite{lazytraining0, lazytraining1, lazytraining2, lazytraining4} and Neural Tangent Kernel (NTK)~\cite{NTK0, NTK1} statements, training loss for a deep model rapidly converges to zero with minimal parameter updates. The neural network is known to behave like its linearization around the initial weight in the entire training process, and the change in Jacobian becomes negligible. 

In the same vein, $f(x + \delta; \theta) \approx f(x; \theta) + \mathbf{J_x} \cdot \delta$. Given a dataset $\mathcal{D} = \{ x_i \}_{i=1}^N$ and $\gamma > 0$, our interest is finding a functional compensatory set $\mathcal{C}_\mathcal{X}$ such that for any $x \in \mathcal{D}$ and $\| \Delta \| \leq \gamma$, there exists $\delta \in \mathcal{C}_\mathcal{X}$ satisfying $f(x; \theta + \Delta) = f(x + \delta ; \theta)$. Utilizing the above two Taylor expansions, the problem reduces to finding the region of $\delta$ satisfying the equation
\begin{equation}\label{jacobian_equality}
\mathbf{J}_{\boldsymbol{\theta}} \cdot \Delta = \mathbf{J_x} \cdot \delta
\end{equation}

A neural network with standard deep learning activations (e.g. ReLU family, sigmoid/tanh family, softplus, GELU, etc.) is said to be in o-minimal structure (or tame) ~\cite{Tame1, Tame2}, which enables us to utilize Sard's theorem. 
By Sard’s theorem~\cite{Sard, MorseSard} stating that the set of critical values has Lebesgue measure zero, the Jacobian matrices have full rank almost everywhere. Therefore, we will assume the matrices $\mathbf{J}_\mathbf{x}$ and $\mathbf{J}_{\boldsymbol{\theta}}$ to be full-rank matrices as in~\cite{jacobianfullrank0, jacobianfullrank1}. Since $c \ll n, p$ and $\mathbf{J}_{\boldsymbol{\theta}}, \mathbf{J_x}$ are full-rank, there always exists a solution $\delta$ for the equation~\eqref{jacobian_equality}. Using psuedoinverse notation $\mathbf{J}_\mathbf{x}^+$ and kernel notation $\text{ker}(\mathbf{J_x})$ for $\mathbf{J_x}$, the solution for equation~\eqref{jacobian_equality} can be written as
$$\delta = \mathbf{J}_\mathbf{x}^+ \cdot \mathbf{J}_{\boldsymbol{\theta}} \cdot \Delta + w$$
, where $w \in \text{ker}(\mathbf{J_x})$. Under the constraint $\| \Delta \| \leq \gamma$, the full feasible solution region can be written as $\mathcal{S}_{\max}(x) := \{ \mathbf{J}_\mathbf{x}^+ \mathbf{J}_{\boldsymbol{\theta}} \, \Delta + w : \| \Delta \| \leq \gamma, w \in \text{ker}(\mathbf{J_x}) \}$, and the minimal covering solution region as $\mathcal{S}_{\min}(x) := \{ \mathbf{J}_\mathbf{x}^+ \mathbf{J}_{\boldsymbol{\theta}} \, \Delta : \| \Delta \| \leq \gamma \} $. By the definition of functional compensatory set, both $\mathcal{S}_{\max}(x)$ and $\mathcal{S}_{\min}(x)$ can be the functional compensatory sets for $x \in \mathcal{D}$ and $\| \Delta \| \leq \gamma$. For the sake of simplicity, we will investigate the feasible region $\mathcal{S}_{\min}(x)$.

There are two things to consider at this point. First, the set $\{ Ax : \| x \| \leq \gamma \}$ is always a rotated ellipsoid, whether the matrix $A$ is degenerate or not. Thus, $\mathcal{S}_{\min}(x) := \{ \mathbf{J}_\mathbf{x}^+ \mathbf{J}_{\boldsymbol{\theta}} \, \Delta : \| \Delta \| \leq \gamma \}$ is an $c$-dimensional rotated ellipsoid having singular values of $\mathbf{J}_\mathbf{x}^+ \mathbf{J}_{\boldsymbol{\theta}}$ as the lengths of its principal semi-axes. Second, any set $\mathcal{S}$ containing $\mathcal{S}_{\min}$ is also a compensatory set since $\mathcal{S}_{\min} \subseteq \mathcal{S}$ implies that for all $\delta \in \mathcal{S}_{\min}$, there exists $ \delta \in \mathcal{S}$. We will derive the closed-form region $\mathcal{S}$ using the fact that given an ellipsoid with maximum possible principal semi-axes length $l_{\max}$, the ellipsoid in encapsulated by an isotropic ball with radius $l_{\max}$ as well.

Let $\mathbf{J_x} = U_x \Sigma_x V_x^T$ and $\mathbf{J}_{\boldsymbol{\theta}} = U_\theta \Sigma_\theta V_\theta^T$ be the singular value decomposition representations, and $\mathbf{J}_\mathbf{x}^+ = V_x \Sigma_x^+ U_x^T$ be the psuedoinverse representation of $\mathbf{J_x}$. Let $\sigma_1^x, ..., \sigma_c^x$ and $\sigma_1^\theta, ..., \sigma_c^\theta$ be the singular values of $\mathbf{J_x}$ and $\mathbf{J}_{\boldsymbol{\theta}}$ (i.e. diagonal elements of $\Sigma_x$ and $\Sigma_\theta$.) Trivially, $\mathbf{J_x}^+$ has singular values $1/\sigma_1^x, ..., 1/\sigma_c^x$.

\begin{align*}
\| \mathbf{J}_\mathbf{x}^+ \mathbf{J}_{\boldsymbol{\theta}} \, \Delta \|
&\leq \| \mathbf{J}_\mathbf{x}^+ \| \cdot \| \mathbf{J}_{\boldsymbol{\theta}} \| \cdot \| \Delta\| && (\because \| A B \| \leq \| A \| \cdot \| B \|) \\
&=\|V_x \Sigma_x^+ U_x^T \| \cdot \|U_\theta \Sigma_\theta V_\theta^T \| \cdot \gamma \\
&\leq \| \Sigma_x^+ \| \| \Sigma_\theta \| \cdot \gamma && (\because \| Q \| = 1 \text{ for orthonormal matrix }Q) \\
&= \dfrac{\max_{i=1, ..., c} \ \sigma_i^\theta}{\min_{i=1, ..., c} \ \sigma_i^x} \cdot \gamma =: \dfrac{\sigma_{\max}^\theta}{\sigma_{\min}^x} \cdot \gamma
\end{align*}

Since the above inequality is about arbitrary datapoint $x \in \mathcal{D}$, we can write inequalities for the whole dataset $\mathcal{D}=\{ x_i \}_{i=1}^N$ as
$$\| \mathbf{J}_\mathbf{x}(x_1)^+ \mathbf{J}_{\boldsymbol{\theta}}(\theta) \, \Delta \| \leq \sigma_{\max}^\theta / \sigma_{\min}^{x_1} \cdot \gamma, \cdots, \| \mathbf{J}_\mathbf{x}(x_N)^+ \mathbf{J}_{\boldsymbol{\theta}}(\theta) \, \Delta \| \leq \sigma_{\max}^\theta / \sigma_{\min}^{x_N} \cdot \gamma.$$

Then, the following bound holds:

$$\text{For all } x \in \mathcal{D} \text{ and } \| \Delta \| \leq \gamma, \quad \| \mathbf{J}_\mathbf{x}^+(x) \mathbf{J}_{\boldsymbol{\theta}}(\theta) \, \Delta \| \leq \max_{x \in \mathcal{D}} \dfrac{\sigma_{\max}^\theta}{\sigma_{\min}^x} \cdot \gamma.$$

Therefore, $\mathcal{S}_\mathcal{X} \;=\; \left\{ \,\delta \in \mathbb{R}^n : \|\delta\| \;\leq\; 
\max_{x \in \mathcal{D}} \left(\dfrac{\sigma_{\max}^\theta}{\sigma_{\min}^x}\right) \cdot \gamma 
\right\}$ is a functional compensatory set for parameter perturbations $\{\Delta \in \mathbb{R}^p : \|\Delta\| \leq \gamma\}$.

\end{proof}

\textbf{Theorem 1 (B).}
\textit{
(Input-to-Parameter) Consider a function $f: \mathbb{R}^n \times \mathbb{R}^p \rightarrow \mathbb{R}^c$. Conversely, for any $\gamma>0$, define the ball of input perturbations
$\{\delta \in \mathbb{R}^n : \|\delta\|\le \gamma\}$ and 
$$\mathcal{S}_\Theta \;=\; \left\{\Delta \in \mathbb{R}^p : \|\Delta\|\;\le\;\max_{x \in \mathcal{D}}\!\Bigl(\dfrac{\sigma_{\max}^{x}}{\sigma_{\min}^\theta}\Bigr)\,\gamma \right\}.$$ Then, $\mathcal{S}_\Theta$ is always a functional compensatory set for the input perturbation ball and dataset $\mathcal{D}$.
}

\begin{proof}
We employ the same strategy to that in the proof of Theorem 1.A. For conciseness, we do not restate the complete set of assumptions and justifications presented previously. Under the first-order Talyor expansion, $f(x; \theta + \Delta) \approx f(x ; \theta) + \mathbf{J}_{\boldsymbol{\theta}} \cdot \Delta$ and $f(x + \delta; \theta) \approx f(x; \theta) + \mathbf{J_x} \cdot \delta$, resulting in $\mathbf{J}_{\boldsymbol{\theta}} \cdot \Delta = \mathbf{J_x} \cdot \delta$. Since $\mathbf{J}_{\boldsymbol{\theta}}$ and $\mathbf{J_x}$ are full-rank, the solution for the equation becomes $$\Delta = \mathbf{J}_{\boldsymbol{\theta}}^+ \cdot \mathbf{J_x} \cdot \delta + w$$, where $w \in \text{ker}(\mathbf{J}_{\boldsymbol{\theta}})$. Under the constraint $\| \delta \| \leq \gamma$, the minimal feasible solution region becomes $\mathcal{S}_{\min}(x) := \{ \mathbf{J}_{\boldsymbol{\theta}}^+ \mathbf{J_x} \, \delta: \| \delta \| \leq \gamma \}$. Under the singular value decomposition, $\mathbf{J}_{\boldsymbol{\theta}} = U_\theta \Sigma_\theta V_\theta^T$ and $\mathbf{J_x} = U_x \Sigma_x V_x^T$. Let $\sigma_1^\theta, ..., \sigma_c^\theta$ and $\sigma_1^x, ..., \sigma_c^x$ represent the singular values of $\mathbf{J}_{\boldsymbol{\theta}}$ and $\mathbf{J_x}$, respectively. Then,

\begin{align*}
\| \mathbf{J}_{\boldsymbol{\theta}}^+ \mathbf{J_x} \, \delta \|
&\leq \| \mathbf{J}_{\boldsymbol{\theta}}^+ \| \cdot \| \mathbf{J_x} \| \cdot \| \delta\| && (\because \| A B \| \leq \| A \| \cdot \| B \|) \\
&=\|V_\theta \Sigma_\theta^+ U_\theta^T \| \cdot \|U_x \Sigma_x V_x^T \| \cdot \gamma \\
&\leq \| \Sigma_\theta^+ \| \| \Sigma_x \| \cdot \gamma && (\because \| Q \| = 1 \text{ for orthonormal matrix }Q) \\
&= \dfrac{\max_{i=1, ..., c} \ \sigma_i^x}{\min_{i=1, ..., c} \ \sigma_i^\theta} \cdot \gamma =: \dfrac{\sigma_{\max}^x}{\sigma_{\min}^\theta} \cdot \gamma
\end{align*}

Since this holds for any $x \in \mathcal{D}$, we can write for $\mathcal{D} = \{ x_i \}_{i=1}^N$:
$$\| \mathbf{J}_\mathbf{\theta}(\theta)^+ \mathbf{J_x}(x_1) \, \delta \| \leq \sigma_{\max}^{x_1} / \sigma_{\min}^\theta \cdot \gamma, \cdots, \| \mathbf{J}_\mathbf{\theta}(\theta)^+ \mathbf{J}_\mathbf{x}(x_N)  \, \delta \| \leq \sigma_{\max}^{x_N} / \sigma_{\min}^{\theta} \cdot \gamma.$$

Leading to:

$$\text{For all } x \in \mathcal{D} \text{ and } \| \delta \| \leq \gamma, \quad \| \mathbf{J}_\mathbf{x}^+(x) \mathbf{J}_{\boldsymbol{\theta}}(\theta) \, \delta \| \leq \max_{x \in \mathcal{D}} \dfrac{\sigma_{\max}^x}{\sigma_{\min}^\theta} \cdot \gamma.$$

Consequently, $\mathcal{S}_\Theta \;=\; \left\{ \,\Delta \in \mathbb{R}^p : \|\Delta\| \;\leq\; 
\max_{x \in \mathcal{D}} \left(\dfrac{\sigma_{\max}^x}{\sigma_{\min}^\theta}\right) \cdot \gamma 
\right\}$ is a functional compensatory set for input perturbations $\{\delta \in \mathbb{R}^n : \|\delta\| \leq \gamma\}$.

\end{proof}

\subsection*{B.2. Proof on Theorem 2}
\label{proof_thm2}

\setcounter{equation}{8}

\begin{theorem}
(Flatness of $\tilde{\theta}^{*}$) Let $\theta^*\in\Theta^*$ and $\tilde{\theta}^{*}\in\tilde{\Theta}^*$ be $b^*$ and $\tilde{b}^*$-flat minima, respectively. The following inequality holds: 
\begin{equation}
\min_{\theta^*\in\Theta^*}b^{*}\leq \min_{\tilde{\theta}^*\in\tilde{\Theta}^*}\tilde{b}^{*}.
\end{equation}
\end{theorem}

\begin{proof}

We consider $f(\cdot; \theta)$ to be a universal approximator for functions from the input space $\mathbb{R}^n$ to the output space $\mathbb{R}^c$.
For the sake of simplicity, we will omit the subscript notations $\theta^*\in\Theta^*$ and $\tilde{\theta}^*\in\tilde{\Theta}^*$ in $\min_{\theta^*\in\Theta^*}$ and $\min_{\tilde{\theta}^*\in\tilde{\Theta}^*}$ henceforth. We prove $\min b^{*} = 0 < 
\left(\max_{x \in \mathcal{D}} \sigma_{\max}^\theta / \sigma_{\min}^x \right)^{-1} \cdot \gamma_\mathcal{A} \leq \min \tilde{b}^{*}.$

Let $\theta^*_0$ be an optimal parameter of a model that satisfies for all $(x, y) \in \mathcal{D},$ $ \mathcal{L}(f(x;\theta_0^*), y)$ $ = 0$ and for all $ (x,y) \in \mathbb{R}^n \times \mathbb{R}^c \  /  \ \mathcal{D},$ $\mathcal{L}(f(x;\theta_0^*), y) > 0$. Let $b_0^*$ denote the $b$-value of the $b$-flat minima for $\theta^*_0$.

\textit{Suppose $b_0^* > 0$}. By the definition of $b$-flat minima, for all $ \| \Delta \| \leq b_0^*$ and $(x, y) \in \mathcal{D}$, $\mathcal{L}(f(x;\tilde{\theta}^*_0 + \Delta), y) = \mathcal{L}(f(x;\tilde{\theta}^*_0), y))$.
\begin{comment}
By Theorem 1, there exists $ \ \delta \in \mathcal{C}_\mathcal{X}$ 
such that $\mathcal{L}(f(x+\delta; \theta_0^*), y)$ $ = \mathcal{L}(f(x;\theta_0^*), y)$ $=0$ and 
$\|\delta\| > 0 $ for some $(x,y) \in \mathcal{D}$.
\end{comment}
Define $\mathcal{C}_\Theta := \{ \Delta  \in \mathbb{R}^p : \| \Delta \| \leq b_0^* \}$. By Theorem~\ref{thm1}, $\mathcal{C}_\Theta$ is a functional compensatory set for input perturbations

$$\left\{ \delta \in \mathbb{R}^n : \left(\max_{x \in \mathcal{D}} \dfrac{\sigma_{\max}^x}{\sigma_{\min}^\theta}\right)^{-1} \cdot b_0^* \right\} =: \left\{ \delta \in \mathbb{R}^n : \lambda \cdot b_0^* \right\}$$

, where $\lambda > 0$. Then, for all $\| \delta \| \leq \lambda \cdot b_0^*, \ \mathcal{L}(f(x + \delta;\theta_0^*), y) = \mathcal{L}(f(x;\theta_0^*), y)) = 0$ by the definition of functional compensatory set. Nonetheless, $\mathcal{L}(f(x+\delta; \theta_0^*), y) > 0$ by the definition of $\theta_0^*$. $\Rightarrow\Leftarrow$

This implies $\min b^* \leq \min b_0^* \leq 0$, i.e. $\min b^* = 0.$

We now show $\left(\max_{x \in \mathcal{D}} \sigma_{\max}^\theta / \sigma_{\min}^x \right)^{-1} \cdot \gamma_\mathcal{A} \leq \min \tilde{b}^*$ as follows:

Let $\tilde{\mathcal{D}} = \{(\mathcal{A}(x), y) : (x, y) \in \mathcal{D} \}$ represent an augmented dataset. Let $\tilde{\theta}^*$ be some optimal parameter which achieves zero loss on the augmented dataset, i.e. for all $ (\tilde{x},y) \in \tilde{\mathcal{D}},$ $\mathcal{L}(f(\tilde{x}; \tilde{\theta}^*), y)) = 0$. Finally, let $\tilde{b}^*$ be the $b$ value for the optimal parameter of $\tilde{\mathcal{D}}$. We show 
$$\left(\max_{x \in \mathcal{D}} \sigma_{\max}^\theta \big/ \sigma_{\min}^x \right)^{-1} \cdot \gamma_\mathcal{A} \leq \min \tilde{b}^*$$
as follows. By definition, $\gamma_\mathcal{A}$ represents the maximum $L_2$ distance that augmentation has nonzero probabilities around the original input. Define $\mathcal{C}_\mathcal{X} = \{ \delta \in \mathbb{R}^n : \| \delta \| \leq \gamma_\mathcal{A} \}.$ Under Theorem 1, $\mathcal{C}_\mathcal{X}$ is the compensatory set for the parameter perturbations

$$\left\{ \Delta \in \mathbb{R}^p : \max_{x \in \mathcal{D}} \left(\dfrac{\sigma_{\max}^x}{\sigma_{\min}^\theta}\right)^{-1} \gamma_\mathcal{A} \right\}.$$

By definition of $b$-flat minima, we have $\left(\max_{x \in \mathcal{D}} \sigma_{\max}^\theta / \sigma_{\min}^x \right)^{-1} \cdot \gamma_\mathcal{A} \leq \min \tilde{b}^*$.
\end{proof}

\newpage

\subsection*{B.3. Proof on Theorem 3}
\label{proof_thm3}

\begin{theorem}
(Generalization bound) Given $M$ covering sets $\{\Theta_k\}_{k=1}^M$ of parameter space $\Theta$ with $\Theta = \bigcup_{k=1}^M \Theta_k$ and $\text{diam}(\Theta) = \sup_{\theta,\theta' \in \Theta} ||\theta - \theta'||_2$, where $M = \left\lceil \frac{\text{diam}(\Theta)}{\gamma_\Theta} \right\rceil^p$, and VC dimension $v_k$ for each $\Theta_k$, the following inequality holds with probability at least $1-\delta$:
\begin{align}
\mathcal{E}_{\mathcal{T}}(\tilde{\theta}^*) &< \hat{\mathcal{E}}_{\mathcal{\tilde{D}}}(\tilde{\theta}^*) + \frac{1}{2} \textbf{Div}(\mathcal{D}, \mathcal{T}) + \max_{k\in [1,M]} \left[ \sqrt{\frac{v_k \log(N/v_k)}{N} + \frac{\log(M/\delta)}{N}} \right], \notag
\end{align}
where $\textbf{Div}(\mathcal{D}, \mathcal{T}) = 2 \sup_A |\mathbb{P}_\mathcal{D}(A) - \mathbb{P}_\mathcal{T}(A)|$ measures the maximal discrepancy between the source and the target distributions $\mathcal{D}$ and $\mathcal{T}$, and $N$ is the number of samples drawn from $\mathcal{\tilde{D}}$.
\end{theorem}

\begin{proof}
Remind that $\mathcal{\tilde{D}}$ represents augmented data distribution, and $\gamma_{\Theta}:=\left(\max_{x \in \mathcal{D}}\sigma_{\max}^\theta / \sigma_{\min}^x \right)^{-1} \cdot \gamma_\mathcal{A}$ the radius of the compensatory set $\mathcal{C}_\Theta$ defined in Theorem~\ref{thm1}. The robust empirical risk is originally defined~\cite{SWAD:2021} as:
$$\hat{\mathcal{E}}_{\mathcal{D}}^\gamma(\theta) := \max_{\|\delta\| \leq \gamma} \frac{1}{N}\sum_{i=1}^N \big[L(f(x_i, \theta + \delta), y_i)\big].$$
Analogously, we define the counterpart in the \textit{input} space as follows:
$$\hat{\mathscr{E}}_{\mathcal{D}}^\gamma(\theta) := \max_{\|\delta\| \leq \gamma} \frac{1}{N}\sum_{i=1}^N \big[L(f(x_i + \delta, \theta), y_i)\big].$$

\noindent For the sake of readability, the notation $\mathcal{B}(M, N, \delta):=\max_{k\in [1,M]} \left[ \sqrt{\frac{v_k \log(N/v_k)}{N} + \frac{\log(M/\delta)}{N}} \right]$ will be used henceforth.

\noindent We now list the necessary propositions from~\cite{SWAD:2021} for the rest of the proof.

\noindent \textbf{Proposition 1.} For all $ \theta \in \Theta, \ |\mathcal{E}_\mathcal{D}(\theta) - \mathcal{E}_\mathcal{T}(\theta)| \leq \dfrac{1}{2} \textbf{Div} (\mathcal{D}, \mathcal{T}).$\\

\noindent \textbf{Proposition 2.} For all $ \theta \in \Theta, \gamma > 0, \ \mathcal{E}_\mathcal{D}(\theta) \leq \hat{\mathcal{E}}_\mathcal{D}^{\gamma}(\theta) + \mathcal{B}(M, N, \theta)$.\\

\noindent We are now ready to prove the rest of our theorem.
\begin{align*}
\mathcal{E}_\mathcal{T}(\tilde{\theta}^*)
&\leq \mathcal{E}_\mathcal{D}(\tilde{\theta}^*) + \dfrac{1}{2}\textbf{Div}(\mathcal{D}, \mathcal{T}) && \text{Prep. 1} \\
&\leq \mathcal{E}_\mathcal{D}^{\gamma_\Theta}(\tilde{\theta}^*) + \dfrac{1}{2}\textbf{Div}(\mathcal{D}, \mathcal{T}) + \mathcal{B}(M, N, \delta) && \text{Prep. 2, $\gamma = \gamma_\Theta$} \\
&\leq \hat{\mathscr{E}}_\mathcal{D}^{\gamma_\mathcal{A}}(\tilde{\theta}^*) + \dfrac{1}{2}\textbf{Div}(\mathcal{D}, \mathcal{T}) + \mathcal{B}(M, N, \delta) && \text{Def. of $\mathscr{E}, \gamma_\Theta,$ Thm. 1} \\
&\leq \hat{\mathcal{E}}_{\tilde{\mathcal{D}}} (\tilde{\theta}^*) + \dfrac{1}{2}\textbf{Div}(\mathcal{D}, \mathcal{T}) + \mathcal{B}(M, N, \delta) && \text{Def. of $\gamma_\mathcal{A}, \, \tilde{\theta}^*$} \\
\end{align*}
\end{proof}

\newpage

\section*{C. Experimental Details and Additional Experiments} \label{exp_apdx}
In this section, we provide a comprehensive account of the experiments corresponding to the tables and figures presented in the main paper, along with supplementary results and discussions that were excluded from the primary manuscript due to page constraints. The additional results include extra robustness evaluations on different backbone networks, clean accuracies of models in Table~\ref{tab:main_table_cifar} and~\ref{tab:main_table_imagenet}, and singular value bounds on Theorem~\ref{thm1}.

\subsection{C.1. Experimental Methodology} \label{exp_detail_apdx}

For all the experiments assessing flatness, mean accuracy, mean corruption error (mCE), and adversarial robustness, we trained three independent models per benchmark and reported the averaged results, with the exception of the ImageNet experiment in Table~\ref{tab:flatness_metric_imagenet} and~\ref{tab:main_table_imagenet}, which was conducted only \textit{once} due to high computational demands.

\subsection{C.2. Flatness Metrics}\label{flatness_explanation_apdx}
\textbf{Descriptions on Metrics.}
$\mu_{\text{PAC-Bayes}}$ represents simplified PAC-Bayesian bound introduced by \citet{PACBayes}. The simplified PAC-Bayesian bound, denoted as $\mu_{\text{PAC-Bayes}}$, is computed as $1/\sigma$, where $\sigma$ is the largest value such that $|\mathcal{E}_\mathcal{D}(\theta + N(0, \sigma^2I)) - \mathcal{E}_\mathcal{D}(\theta)| \leq \tau$. Therefore, flatter local minima correspond to larger $\sigma$ values and smaller $\mu_{\text{PAC-Bayes}}$. The LPF-based flatness measure evaluates the flatness of a local minimum $\theta^*$ by computing the convolution of the loss function $\mathcal{L}$ with a Gaussian kernel $K=N(0, \sigma^2I)$. Formally, it is defined as $(\mathcal{L} * K)(\theta^*) = \int\mathcal{L}(\theta^* - \tau)K(\tau)d\tau$. This measure effectively averages the loss over a neighborhood around $\theta^*$, with the Gaussian kernel weighting nearby points more heavily, thereby quantifying how sharp the loss landscape is in that region. $\epsilon$-sharpness ($\epsilon_{\text{sharp}}$) measures the sensitivity of a loss function near a local minimum by quantifying the largest perturbation of model parameters that leads to an increase in loss of more than $\epsilon$.
We have adopted the default parameters for the $\mu_{\text{PAC-Bayes}}$, LPF measure, and $\epsilon_{\text{sharp}}$ in CIFAR experiments i.e. $\tau=0.05$, $\sigma=0.01$, and $\epsilon=0.1$ respectively. For the larger (tiny)ImageNet runs, we related parameters $\tau = 0.35$ and $\epsilon = 1.0$ to maintain numerical stability, making the constraints looser than those used for CIFAR.

\noindent \textbf{Training Details.} Our training procedure on CIFAR experiments closely follows~\citet{AugMix9:2021}, with the only difference being that we train the WideResNet-40-2 model for 200 epochs instead of 100. We employ an SGD optimizer with an initial learning rate of 0.1, a weight decay of 0.0005, and a momentum of 0.9. Additionally, we use a cosine learning rate decay scheduler that starts with the learning rate of 0.1 and a weight decay of 0.0005. Finally, we evaluate and report the flatness of the converged models. For \textit{tiny}ImageNet and ImageNet experiments, we have majorly adopted the configurations from~\cite{AugMix9:2021}, with changes in epochs and backbone networks for faster training. We have trained ResNet18 on \textit{tiny}ImageNet and ImageNet for 50 epochs, instead of training ResNet50 for 90 epochs. SGD optimizer with initial learning rate 0.01, weight decay 0.0005, and momentum 0.9 has been adopted. All the models have been trained until convergence, with cosine learning rate decaying scheduling. Regarding the augmentation techniques, default configuration parameters were used for all baseline augmentations except StyleAug, which had hyperparameter $\alpha$ set to $0.95$ for the CIFAR experiments, and 1.0 for the \textit{tiny}ImageNet experiments.

\subsection{C.3. Robustness Benchmark Details} \label{benchmark_explanation_apdx}
\begin{comment}
\textbf{Domain Generalization Benchmarks.} We conducted comprehensive domain generalization experiments on the PACS and OfficeHome datasets to evaluate model robustness and generalizability using augmentations. These datasets are specifically designed to test a model's ability to generalize across different domains, as they consist of images that share the same labels but differ in their domain representations. For example, the PACS dataset contains images from four distinct domains—'art painting', 'cartoon', 'photo', and 'sketch'—all featuring the same classes such as 'dog', 'elephant', and others. In all experiments, models were trained until convergence to ensure fair and consistent comparisons across different settings.

Our baseline employed Empirical Risk Minimization (ERM) with the best hyperparameters selected from DomainBed's ERM training configuration \citep{DomainBed}. To specifically assess the impact of augmentations on model performance, we excluded the heavy augmentation compositions used in the original DomainBed configuration.

We have adopted 2,500 steps of training iterations, which results in slightly faster training than the default setting of 5,000. Specifically for StyleAug, we performed additional tuning of its hyperparameter $\alpha$ alongside the training epochs. Through manual experimentation, we found that setting $\alpha=1.0$ yielded the best results in terms of model robustness and generalization performance.
\end{comment}

\noindent \textbf{Common Corruption Benchmarks.} We assessed the common corruption robustness of models trained with augmentations on the CIFAR-10/100, \textit{tiny}ImageNet, and ImageNet datasets in Appendix C.2. Specifically, we utilized their well-known corrupted counterparts---CIFAR-10/100-C/$\overline{\text{C}}$, \textit{tiny}ImageNet-C, and ImageNet-C--- which are standard benchmarks generated from the original datasets by applying 15 distinct corruption types at 5 different severity levels. These corruption types include brightness changes, contrast alterations, defocus blur, elastic transformations, fog addition, frost addition, Gaussian blur, glass distortion, impulse noise, JPEG compression, motion blur, pixelation, shot noise, snow addition, and zoom blur. Models were trained and validated on the respective clean datasets, with robustness evaluations performed at test time only to simulate real-world scenarios where models encounter unforeseen corruptions. To quantify model performance across all corruptions and severity levels, we calculated the mean Corruption Error (mCE), which averages the error rates over the different conditions.

\noindent \textbf{Adversarial Robustness Benchmarks.}
To benchmark model robustness under adversarial threat, we evaluated the models trained with augmentations on CIFAR-10/100, \textit{tiny}ImageNet, and ImageNet using widely recognized, untargeted Projected Gradient Descent (PGD) attacks crafted under both $L_2$ and $L_\infty$ norms. 
The training environment holds the same as in the common corruption benchmarks. We will use $\epsilon$ and $\alpha$ to denote the maximum allowed perturbation size. Regarding CIFAR-10/100 experiments, we utilized attack configurations from~\cite{boundary_thickness}. Specifically, we deployed PGD-20 $L_2$ attacks with $\epsilon=0.5, \alpha=0.0125$ and PGD-7 $L_\infty$ attacks with $\epsilon=8/255, \alpha=2/255$. For \textit{tiny}ImageNet experiments, we utilized PGD-10 $L_2$ attacks with $\epsilon=0.25, \alpha=0.025$ and PGD-5 $L_\infty$ attacks with $\epsilon=8/255, \alpha=2/255$. On ImageNet experiments, we deployed PGD-10 $L_2$ attacks with $\epsilon=0.25, \alpha=0.025$ and PGD-2 $L_\infty$ attacks with $\epsilon=2/255, \alpha=1/255$. In essence, milder attack configurations have been deployed for \textit{tiny}ImageNet and ImageNet.

\subsection*{C.4. Additional Experiments on Varying Backbone Networks}

To confirm our theorems across different backbone architectures, we provide supplementary experiments on CIFAR-100 robustness benchmarks, expanding upon the data shown in Table~\ref{tab:main_table_cifar}. On CIFAR-100, augmentations like AugMix, RandAug, and PixMix closely satisfy the PSA condition by producing significant augmented samples in the neighborhood of the original data, which results in flatter minima and improved robustness. In contrast, DeepAug and StyleAug lack augmented samples near the original inputs, failing to satisfy both the PSA condition and improved robustness. This pattern remains consistent across different backbone networks. As illustrated in Table~\ref{tab:varying_backbone}, our conclusions remain consistent on distinct backbone architectures~\cite{allconvnet, densenet, resnext}.

\begin{table*}[ht]
\centering
\begin{adjustbox}{width=1.0\textwidth}
\begin{tabular}{c|c|c|ccccc}
\midrule[1.5pt]
\textbf{Backbone} & \textbf{Benchmark} & ERM & $ \ \ $ AugMix $ \ \ $ & $ \ \ $ RandAug $ \ \ $ & $ \ \ $ PixMix $ \ \ $ & $ \ \ $ DeepAug $ \ \ $ & $ \ \ $ StyleAug $ \ \ $ \\
\midrule[1.5pt]

\multirow{4}{*}{AllConvNet} & \small CIFAR-100-C $\downarrow$
& 56.88 & 
\textbf{43.13} \small (-13.75) & 
\textbf{47.52} \small (-9.36) & 
\textbf{37.70} \small (-19.18) & 
\emph{62.00} \small (+5.12) & 
\emph{56.90} \small (+0.02) \\

 & \small CIFAR-100-$\overline{\text{C}}$ $\downarrow$
& 20.33 & 
\textbf{9.84} \small (-10.49) & 
\textbf{13.86} \small (-6.47) & 
\textbf{14.29} \small (-6.04) & 
\emph{64.24} \small (+43.91) & 
\emph{38.64} \small (+18.31) \\

& \small CIFAR-100, $L_{2}$ $\downarrow$
& 93.96 & 
\textbf{86.51} \small (-7.45) & 
\textbf{88.96} \small (-5.00) & 
\textbf{74.93} \small (-19.03) & 
\emph{95.29} \small (+1.33) & 
\emph{96.56} \small (+2.60) \\

 & \small CIFAR-100, $L_{\infty}$ $\downarrow$
& 99.95 & 
\textbf{99.90} \small (-0.05) & 
\textbf{99.94} \small (-0.01) & 
\textbf{99.64} \small (-0.31) & 
\emph{99.99} \small (+0.04) & 
\textbf{99.71} \small (-0.24) \\

\midrule

\multirow{4}{*}{DenseNet} & \small CIFAR-100-C $\downarrow$
& 61.21 & 
\textbf{43.32} \small (-17.89) & 
\textbf{48.54} \small (-12.67) & 
\textbf{37.58} \small (-23.63) & 
\emph{61.47} \small (+0.26) & 
\emph{68.97} \small (+7.76) \\

 & \small CIFAR-100-$\overline{\text{C}}$ $\downarrow$
& 23.10 & 
\textbf{8.87} \small (-14.23) & 
\textbf{13.68} \small (-9.42) & 
\textbf{14.27} \small (-8.83) & 
\emph{66.20} \small (+43.10) & 
\emph{43.90} \small (+20.80) \\

& \small CIFAR-100, $L_{2}$ $\downarrow$
& 99.91 & 
\textbf{97.88} \small (-2.03) & 
\textbf{99.44} \small (-0.47) & 
\textbf{93.19} \small (-6.72) & 
\textbf{98.70} \small (-1.21) & 
\textbf{99.76} \small (-0.15) \\

 & \small CIFAR-100, $L_{\infty}$ $\downarrow$
& 100.00 & 
\textbf{99.99} \small (-0.01) & 
\textbf{99.99} \small (-0.01) & 
\textbf{99.99} \small (-0.01) & 
\textbf{99.99} \small (-0.01) & 
\textbf{99.99} \small (-0.01) \\

\midrule

\multirow{4}{*}{ResNext29} & \small CIFAR-100-C $\downarrow$
& 58.05 & 
\textbf{39.60} \small (-18.45) & 
\textbf{43.49} \small (-14.56) & 
\textbf{31.35} \small (-26.70) & 
\emph{61.50} \small (+3.45) & 
\emph{70.11} \small (+12.06) \\

 & \small CIFAR-100-$\overline{\text{C}}$ $\downarrow$
& 13.09 & 
\textbf{3.85} \small (-9.24) & 
\textbf{6.99} \small (-6.10) & 
\textbf{2.64} \small (-10.45) & 
\emph{66.61} \small (+53.52) & 
\emph{22.78} \small (+9.69) \\

& \small CIFAR-100, $L_{2}$ $\downarrow$
& 98.72 & 
\textbf{91.68} \small (-7.04) & 
\textbf{95.95} \small (-2.77) & 
\textbf{91.64} \small (-7.08) & 
\textbf{98.36} \small (-0.36) & 
\emph{99.67} \small (+0.95) \\

 & \small CIFAR-100, $L_{\infty}$ $\downarrow$
& 100.00 & 
\textbf{99.94} \small (-0.06) & 
\textbf{99.97} \small (-0.03) & 
\textbf{99.90} \small (-0.10) & 
\textbf{99.96} \small (-0.04) & 
\textbf{99.98} \small (-0.02) \\

\midrule[1.5pt]
\end{tabular}
\end{adjustbox}

\caption{Comparison of various augmentation methods against ERM across different backbone networks on CIFAR-100 robustness benchmarks. \textit{($\downarrow$: The lower, the better.)}}
\label{tab:varying_backbone}
\end{table*}

\subsection{C.5. Results on Theorem 1's Bound}
Using the CIFAR-100 dataset, we evaluated the singular values for Theorem 1 and obtained
$\sigma_{\text{min}}^{x}=0.20,\sigma_{\text{max}}^{x}=13.48,\sigma_{\text{min}}^{\theta}=4.55,\sigma_{\text{max}}^{\theta}=52.56$.
The minimum singular values are not excessively small and maximum values not excessively large, thereby preventing loose bounds. % \textcolor{red}{(+ singular values on CIFAR-10/(tiny)ImageNet possibly?)}

\subsection*{C.6. Clean Accuracies of Models (Table~\ref{tab:main_table_cifar},~\ref{tab:flatness_metric_imagenet})}

\begin{table}[H]
\centering
\begin{adjustbox}{width=1.0\textwidth}
\begin{tabular}{c|c|ccccc}
\midrule[1.5pt]
$\quad$ \textbf{Dataset} $\quad$ & $ \ \ $ ERM $ \ \ $ & $ \ \ \quad $ AugMix $ \ \ \quad $ & $ \ \ \quad $ RandAug $ \ \ \quad $ & $ \ \ \quad $ PixMix $ \ \ \quad $ & $ \ \ \quad $ DeepAug $ \ \ \quad $ & $ \ \ \quad $ StyleAug $ \ \ \quad $ \\
\midrule[1.5pt]

ImageNet $\uparrow$ & 69.48 & \emph{69.33} \small (-0.15)
                                & \emph{68.85} \small (-0.63)
                                & \emph{66.86} \small (-2.62)
                                & \emph{65.05} \small (-4.43)
                                & \emph{54.35} \small (-15.13) \\
\textit{tiny}ImageNet $\uparrow$ & 64.74 & \textbf{65.08} \small (+0.34)
                                & \textbf{70.09} \small (+5.35)
                                & \textbf{66.97} \small (+2.23)
                                & \textbf{69.50} \small (+4.76)
                                & \emph{50.60} \small (-14.14) \\
CIFAR-10 $\uparrow$ & 90.25 & \textbf{95.60} \small (+5.35)
                                & \textbf{95.77} \small (+5.52)
                                & \textbf{95.75} \small (+5.50)
                                & \emph{80.92} \small (-9.33)
                                & \textbf{90.59} \small (+0.34) \\
CIFAR-100 $\uparrow$ & 64.82 & \textbf{70.40} \small (+5.58)
                                & \textbf{74.20} \small (+9.38)
                                & \textbf{74.68} \small (+9.86)
                                & \emph{49.25} \small (-15.57)
                                & \emph{38.76} \small (-26.06) \\

\midrule[1.5pt]
\end{tabular}
\end{adjustbox}

\end{table}

There is a weak correlation between \textit{robustness} and the \textit{clean accuracy}.
For instance, all methods have inferior clean accuracy compared to ERM in ImageNet, irrespective of their robustness. There are some cases where the clean accuracy increases but the robustness drops (e.g., StyleAug in CIFAR10), and some cases have dropped clean accuracy but improved robustness (e.g., PixMix in ImageNet). Also, some have severely dropped clean accuracy and robustness (e.g., StyleAug in CIFAR-100). We observe all four cases (better clean accuracy, better robustness / better clean accuracy, worse robustness / worse clean accuracy, better robustness / worse clean acc, worse robustness) exist.

\subsection*{C.7. Confidence Intervals for the Main Tables}

\begin{table*}[h]
\centering
\begin{adjustbox}{width=1.0\textwidth}
\begin{tabular}{cc|c|ccccc}
\midrule[1.5pt]
\textbf{\ \ Dataset \ \ } & \textbf{ \ \ Metrics \ \ } & $ \ \ \ $ ERM $ \ \ \ $ & $ \quad \ $ AugMix $ \quad \ $ & $ \quad \ $ RandAug $ \quad \ $ & $ \quad \ $ PixMix $ \quad \ $ & $ \quad \ $ DeepAug $ \quad \ $ & $ \quad \ $ StyleAug $ \quad \ $ \\
\midrule[1.5pt]

\multirow{3}{*}{CIFAR-10} 
& $\mu_{\text{PAC-Bayes}}\downarrow$ & 
168.83 $\pm$ 13.50 &
\textbf{117.16} $\pm$ 3.15 &
\textbf{110.80} $\pm$ 0.55 &
\textbf{102.87} $\pm$ 4.18 &
\textbf{67.74} $\pm$ 3.27 &
\textbf{122.80} $\pm$ 11.68 \\

& LPF $\downarrow$ &
1.43 $\pm$ 0.29 &
\textbf{0.54} $\pm$ 0.01 &
\textbf{0.45} $\pm$ 0.03 &
\textbf{0.37} $\pm$ 0.01 &
\textbf{0.72} $\pm$ 0.10 &
\textbf{1.13} $\pm$ 0.12 \\

&$\epsilon_{\text{sharp}}\downarrow$ & 40.90 $\pm$ 3.90 & 
\textbf{25.16} $\pm$ 4.32 &
\textbf{24.47} $\pm$ 2.81 &
\textbf{26.57} $\pm$ 8.34 &
\textit{55.63} $\pm$ 34.78 &
\textit{130.64} $\pm$ 1.69 \\

\midrule

\multirow{3}{*}{CIFAR-100} 

& $\mu_{\text{PAC-Bayes}}\downarrow$ & 
181.10 $\pm$ 9.53 &
\textbf{155.74} $\pm$ 0.01 &
\textbf{149.79} $\pm$ 8.41 &
\textbf{134.58} $\pm$ 1.64 &
\textbf{72.24} $\pm$ 2.47 &
\textbf{129.02} $\pm$ 13.25 \\

& LPF $\downarrow$ & 
2.37 $\pm$ 0.16 &
\textbf{1.84} $\pm$ 0.01 &
\textbf{1.69} $\pm$ 0.02 &
\textbf{1.42} $\pm$ 0.03 &
2.37 $\pm$ 0.23 &
\textit{3.95} $\pm$ 0.14 \\

& $\epsilon_{\text{sharp}}\downarrow$ & 
42.39 $\pm$ 4.68 &
\textbf{39.32} $\pm$ 2.93 &
\textbf{37.76} $\pm$ 3.32 &
\textbf{33.14} $\pm$ 4.76 &
\textit{84.48} $\pm$ 4.33 &
\textit{302.81} $\pm$ 5.14 \\

\midrule[1.5pt]
\end{tabular}
\end{adjustbox}

\begin{adjustbox}{width=0.98\textwidth}
\begin{tabular}{c|c|ccccc}
\midrule[1.5pt]
\textbf{$\quad \ $ Benchmarks $\quad \ $} & \ \ ERM \ \ & 
$\quad \ $ AugMix $\quad \ $ & $\quad \ $ RandAug $\quad \ $ & $\quad \ $ PixMix $\quad \ $ & $\quad \ $ DeepAug $\quad \ $ & $\quad \ $ StyleAug $\quad \ $ \\
\midrule[1.5pt]

\small CIFAR-10-C $\downarrow$
& 30.54 $\pm$ 0.45 &
\textbf{15.24} $\pm$ 0.27 &
\textbf{19.65} $\pm$ 0.32 &
\textbf{10.60} $\pm$ 0.18 &
\textit{31.92} $\pm$ 3.90 &
\textbf{30.50} $\pm$ 1.00 \\

\small CIFAR-10-$\overline{\text{C}}$ $\downarrow$
& 31.35 $\pm$ 0.59 &
\textbf{20.28} $\pm$ 0.32 &
\textbf{20.64} $\pm$ 0.14 &
\textbf{14.60} $\pm$ 0.16 &
\textit{36.75} $\pm$ 2.43 &
\textit{36.57} $\pm$ 0.02 \\

\small CIFAR-100-C $\downarrow$
& 59.04 $\pm$ 0.45 &
\textbf{42.64} $\pm$ 0.16 &
\textbf{46.59} $\pm$ 0.74 &
\textbf{35.20} $\pm$ 0.22 &
\textit{62.34} $\pm$ 3.33 &
\textit{70.91} $\pm$ 0.88 \\

\small CIFAR-100-$\overline{\text{C}}$ $\downarrow$
& 62.43 $\pm$ 1.16 &
\textbf{48.38} $\pm$ 0.27 &
\textbf{48.32} $\pm$ 0.41 &
\textbf{40.20} $\pm$ 0.18 &
\textit{67.91} $\pm$ 2.31 &
\textit{76.56} $\pm$ 0.57 \\

\small CIFAR-10, $L_2$ $\downarrow$
& 77.61 $\pm$ 3.86 &
\textbf{70.76} $\pm$ 1.05 &
\textbf{76.15} $\pm$ 1.65 &
\textbf{65.81} $\pm$ 2.07 &
\textit{91.18} $\pm$ 1.80 &
\textit{91.06} $\pm$ 0.39 \\

\small CIFAR-10, $L_\infty$ $\downarrow$
& 98.49 $\pm$ 1.64 &
\textit{99.10} $\pm$ 0.16 &
\textit{99.86} $\pm$ 0.02 &
\textit{99.46} $\pm$ 0.26 &
\textit{100.00} $\pm$ 0.00 &
\textit{99.98} $\pm$ 0.03 \\

\small CIFAR-100, $L_2$ $\downarrow$
& 98.73 $\pm$ 0.62 &
\textbf{92.76} $\pm$ 0.19 &
\textbf{96.06} $\pm$ 1.85 &
\textbf{90.69} $\pm$ 2.55 &
\textbf{98.44} $\pm$ 0.50 &
\textit{99.69} $\pm$ 0.02 \\

\small CIFAR-100, $L_\infty$ $\downarrow$
& 99.94 $\pm$ 0.09 &
\textbf{99.67} $\pm$ 0.09 &
99.94 $\pm$ 0.11 &
\textit{99.69} $\pm$ 0.44 &
\textit{99.99} $\pm$ 0.02 &
\textit{99.98} $\pm$ 0.02 \\

\midrule[1.5pt]
\end{tabular}
\end{adjustbox}

\begin{adjustbox}{width=0.98\textwidth}
\begin{tabular}{c|c|ccccc}
\midrule[1.5pt]
\textbf{$\quad \ $ Benchmarks $\quad \ $} & \ \ ERM \ \ & 
$\quad \ $ AugMix $\quad \ $ & $\quad \ $ DeepAug $\quad \ $ & $\quad \ $ PixMix $\quad \ $ & $\quad \ $ RandAug $\quad \ $ & $\quad \ $ StyleAug $\quad \ $ \\
\midrule[1.5pt]

\small \textit{tiny}ImageNet-C $\downarrow$
& 74.76 $\pm$ 0.47 &
\textbf{63.48} $\pm$ 0.28 &
\textbf{56.14} $\pm$ 0.12 &
\textbf{61.51} $\pm$ 0.28 &
\textbf{66.92} $\pm$ 0.21 &
\textit{79.42} $\pm$ 1.19 \\

\small \textit{tiny}ImageNet, $L_2$ $\downarrow$
& 57.77 $\pm$ 0.57 &
\textit{61.84} $\pm$ 1.13 &
\textbf{52.86} $\pm$ 0.34 &
\textit{58.38} $\pm$ 1.41 &
\textit{63.19} $\pm$ 0.74 &
\textit{83.13} $\pm$ 0.63 \\

\small \textit{tiny}ImageNet, $L_\infty$ $\downarrow$
& 99.94 $\pm$ 0.02&
\textbf{99.93} $\pm$ 0.04 &
\textit{99.95} $\pm$ 0.03 &
\textit{99.98} $\pm$ 0.02 &
\textit{99.98} $\pm$ 0.01 &
\textit{100.0} $\pm$ 0.00 \\

\midrule[1.5pt]
\end{tabular}
\end{adjustbox}

\begin{adjustbox}{width=1.0\textwidth}
\begin{tabular}{cc|c|ccccc}
\midrule[1.5pt]
\textbf{\ \ Dataset \ \ } & \textbf{ \ \ Metrics \ \ } & $ \ \ \ $ ERM $ \ \ \ $ & $ \quad \ $ AugMix $ \quad \ $ & $ \quad \ $ DeepAug $ \quad \ $ & $ \quad \ $ PixMix $ \quad \ $ & $ \quad \ $ RandAug $ \quad \ $ & $ \quad \ $ StyleAug $ \quad \ $ \\
\midrule[1.5pt]

\multirow{3}{*}{\textit{tiny}ImageNet} 
& $\mu_{\text{PAC-Bayes}}\downarrow$ & 
136.83 $\pm$ 2.05 &
\textbf{109.18} $\pm$ 0.97 &
\textbf{128.41} $\pm$ 2.36 &
\textbf{106.26} $\pm$ 1.97 &
\textbf{109.36} $\pm$ 0.97 &
\textbf{102.72} $\pm$ 1.45 \\

& LPF $\downarrow$ &
5.95 $\pm$ 0.27 &
\textbf{3.81} $\pm$ 0.17 &
\textbf{4.96} $\pm$ 0.12 &
\textbf{3.40} $\pm$ 0.09 &
\textbf{4.04} $\pm$ 0.11 &
\textbf{4.33} $\pm$ 0.20 \\

&$\epsilon_{\text{sharp}}\downarrow$ & 
15.35 $\pm$ 0.51 &
\textbf{13.09} $\pm$ 0.32 &
\textit{25.02} $\pm$ 1.11 &
\textbf{10.72} $\pm$ 0.66 &
\textit{21.30} $\pm$ 1.92 &
\textit{35.40} $\pm$ 3.75 \\

\midrule[1.5pt]
\end{tabular}
\end{adjustbox}

\caption{95\% confidence intervals for CIFAR and \textit{tiny}ImageNet experiments. (Table~\ref{tab:flatness_metric_cifar},~\ref{tab:main_table_cifar},~\ref{tab:flatness_metric_imagenet},~\ref{tab:main_table_imagenet}.)}
\end{table*}

\subsection{C.8. Licensing and Computational Resources}
\subsubsection{Licenses}\label{sec:license}
AugMix and RandAugment are under Apache 2.0 License. PixMix, DeepAugment, and StyleAugment are under MIT License.

\subsubsection{Computational Resources for Experiments}
\label{sec:resource}
We have utilized NVIDIA RTX A5000 24G, RTX A6000 48G, and A100 SXM 40G in our experiments. Specifically, we have mainly used A5000 for CIFAR experiments (Table~\ref{tab:flatness_metric_cifar},~\ref{tab:main_table_cifar}), A6000 for \textit{tiny}ImageNet experiments, and A100 SXM 40G for ImageNet experiments. Regarding Table~\ref{tab:main_table_imagenet}, while the training time heavily differed by each augmentation method, the training time for CIFAR experiments took less than three days to compute the average values of each method. The evaluation time of each model on CIFAR took less than half an hour on both common corruption and adversarial robustness benchmarks.
For the \textit{tiny}ImageNet experiments, the training time took roughly 3 days to 1 week per each method running thrice. The evaluation time took roughly one hour per each method. For the ImageNet study, we have deployed external A100 SXM 40G GPUs for evaluation (\texttt{lambdalab}). The training time took roughly 3 days to 1 week per each method running once. The evaluation took roughly 3$\sim$5 hours for the common corruptions, and an hour for the adversarial attacks. 
Regarding Table~\ref{tab:flatness_metric_cifar} and Table~\ref{tab:flatness_metric_imagenet}, the flatness evaluation time took roughly less than two hours for models trained with each method.

\section*{D. Further Discussions}\label{sec:discussion_apdx}
\subsection{Practical Applications of Our Theorems}
Although the theoretical insights we provide into how data augmentation influences robustness are valuable in their own right, we herein present some practical consequences of our findings.

First, when proposing a novel, complex augmentation, prioritize \textit{densely populating} the proximal region around the original image. In practice, researchers can perform a quick, low-cost check by inspecting how the augmented samples cluster near the source image. If that neighborhood is well covered, the resulting model is highly likely to remain robust under unforeseen data shifts.

Second, we can roughly forecast the benefit of combining two augmentations by looking at each augmentation’s sample density close to the original data. For example, DeepAug and StyleAug both leave the region empty, whereas AugMix fills it densely. Hence, pairing AugMix with either DeepAug or StyleAug should yield a substantial improvement (by filling the gap), while combining DeepAug with StyleAug alone should offer little extra robustness.

\begin{figure}[h]
	\begin{center}
		\includegraphics[width=1.0\linewidth]{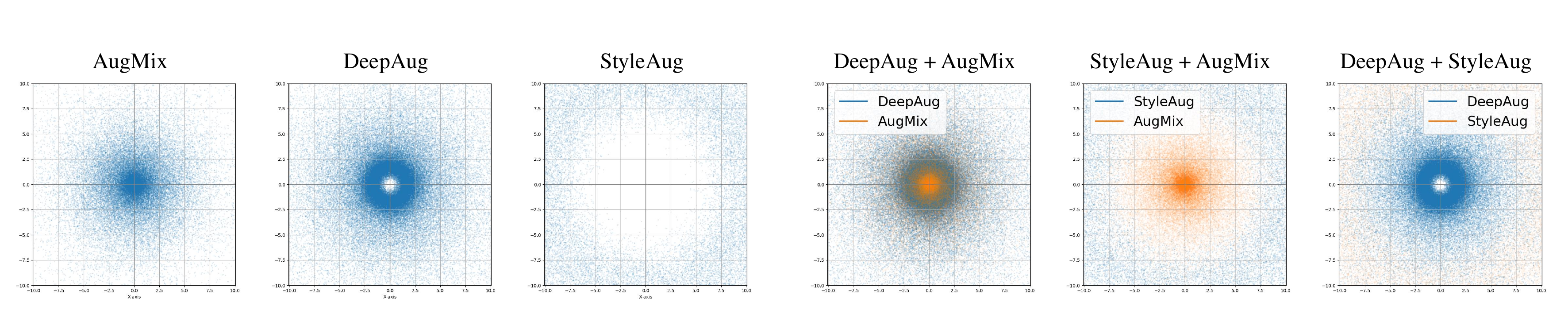}
	\end{center}
	\caption{A 2D sample densities generated by each augmentation and their pairwise combinations. AugMix densely covers the proximal neighborhood, DeepAug and StyleAug leave it sparse, and mixing AugMix with either closes the gap. However, DeepAug + StyleAug still leaves a central hole.}
\label{fig:2d}
\end{figure}

\begin{table*}[ht]
\centering
\begin{adjustbox}{width=1.0\textwidth}
\begin{tabular}{c|cccccc}
\midrule[1.5pt]
\textbf{Benchmark} & AugMix  &  DeepAug  &  StyleAug  &  DeepAug + AugMix  &  StyleAug + AugMix  &  DeepAug + StyleAug  \\
\midrule[1.5pt]

\small CIFAR-10-C $\downarrow$ &
15.24 & 
31.92 & 
30.50 &
\textbf{15.04} \small (-16.88) &
\textbf{19.02} \small (-11.48) &
\textbf{24.84} \small (-5.66) \\

\small CIFAR-10-$\overline{\text{C}}$ $\downarrow$ &
1.73 & 
17.99 & 
9.34 &
\textbf{1.32} \small (-16.67) &
\textbf{1.55} \small (-7.79) &
\textbf{4.12} \small (-5.22)\\

\small CIFAR-100-C $\downarrow$ &
42.64 & 
62.34 & 
70.92 &
\textbf{43.53} \small (-18.81) &
\textbf{49.13} \small (-21.79) &
\textbf{53.50} \small (-17.42)\\

\small CIFAR-100-$\overline{\text{C}}$ $\downarrow$ &
5.01 & 
39.54 & 
30.77 &
\textbf{3.67} \small (-35.87) &
\textbf{3.70} \small (-27.07) &
\textbf{7.60} \small (-23.17) \\

\small CIFAR-10, $L_2$ $\downarrow$ &
70.68 & 
91.07 & 
91.13 &
\textbf{69.30} \small (-21.77) &
\textbf{81.07} \small (-10.06) &
\textbf{87.30} \small (-3.83)\\

\small CIFAR-10, $L_\infty$ $\downarrow$ &
99.57 & 
100.00 & 
99.97 &
\textbf{94.15} \small (-4.22) &
\textbf{96.34} \small (-3.33) &
\textbf{99.80} \small (-0.17) \\

\small CIFAR-100, $L_2$ $\downarrow$ &
92.78 & 
98.37 & 
99.67 &
\textbf{94.15} \small (-4.22) &
\textbf{96.34} \small (-3.33) &
\textbf{97.53} \small (-2.14) \\

\small CIFAR-100, $L_\infty$ $\downarrow$ &
99.67 & 
99.99 & 
99.98 &
\textbf{99.89} \small (-0.10) &
\textbf{99.95} \small (-0.03) &
\textbf{99.96} \small (-0.02) \\

\midrule[1.5pt]
\end{tabular}
\end{adjustbox}
\caption{The effect of combining augmentations that have significant differences in proximal density. In all the cases, combining DeepAug or StyleAug with AugMix yielded notable robustness gain. Nevertheless, combining DeepAug with StyleAug yielded marginal improvement.}
\end{table*}

Our study also provides explanations on the behavior of adversarial training (AT), and rooms for the improvement. While AT is known to strengthen robustness to adversarial attacks, models trained with AT still break down when the test distribution deviates largely from the training data, as in common corruption case~\cite{IN-C7:2019, IN-C-bar:2021}. We conjecture that this fragility arises because AT typically discovers \textit{small} $b$-flat minima: these minima are flat enough to absorb minor, norm-bounded perturbations, yet they fail to cope with large distribution shifts. Hence, pairing AT with data augmentation schemes that generate samples both close to and far from the original data manifold (such as AugMix) can offset AT’s weakness under severe shifts (Table~\ref{tab:fgsm_flatness},~\ref{tab:fgsm_robustness}.)

\subsection{Beyond the Visual Domain}\label{discussion_beyond_visual_data}
Our experiments focus on image classification, yet the PSA condition and accompanying theorems do not rely on properties specific to images. We therefore expect the framework to extend to other modalities such as language, audio, or sensor data, where augmentation is commonplace. Applying the analysis to sequential tasks such as NLP or reinforcement learning is a natural next step and may reveal modality-specific nuances; we leave these directions for future work.

\end{document}